\documentclass[11pt]{article}

\usepackage{dsfont}

\def\colorful{1}

\usepackage{amsthm,amsfonts,amsmath,amssymb,epsfig,color,float,graphicx,verbatim,enumitem}
\usepackage{multirow}
\usepackage[noend]{algpseudocode}
\usepackage{enumitem}
\usepackage[linesnumbered,lined,boxed,commentsnumbered,ruled,vlined]{algorithm2e}

\newif\ifhyper\IfFileExists{hyperref.sty}{\hypertrue}{\hyperfalse}
\hypertrue
\ifhyper\usepackage{hyperref}\fi

\usepackage{enumitem}

\usepackage{framed}
\usepackage{nicefrac}

\usepackage{bm}

\def\nnewcolor{1}
\ifnum\nnewcolor=1

\fi
\ifnum\nnewcolor=0

\fi

\ifnum\colorful=1

\else

\fi

\newtheorem{theorem}{Theorem}[section]

\newtheorem{cond}[theorem]{Condition}
\newtheorem{lemma}[theorem]{Lemma}
\newtheorem{informal theorem}[theorem]{Theorem (informal statement)}

\newtheorem{proposition}[theorem]{Proposition}
\newtheorem{corollary}[theorem]{Corollary}
\newtheorem{claim}[theorem]{Claim}
\newtheorem{fact}[theorem]{Fact}

\newtheorem{remark}[theorem]{Remark}

\theoremstyle{definition}
\newtheorem{definition}[theorem]{Definition}
\newcommand{\eqdef}{\stackrel{{\mathrm {\footnotesize def}}}{=}}

\newcommand{\bx}{\mathbf{x}}
\newcommand{\by}{\mathbf{y}}

\newcommand{\bv}{\mathbf{v}}
\newcommand{\bu}{\mathbf{u}}
\newcommand{\bw}{\mathbf{w}}

\newcommand{\bt}{\mathbf{t}}

\newcommand{\bI}{\mathbf{I}}
\newcommand{\bT}{\mathbf{T}}

\newcommand{\bV}{\mathbf{V}}

\newcommand{\p}{\mathbf{P}}
\newcommand{\q}{\mathbf{Q}}

\newcommand{\R}{\mathbb{R}}

\newcommand{\Z}{\mathbb{Z}}
\newcommand{\N}{\mathbb{N}}
\newcommand{\E}{\mathbf{E}}

\newcommand{\dtv}{d_{\mathrm{TV}}}
\newcommand{\pr}{\mathbf{Pr}}

\newcommand{\poly}{\mathrm{poly}}

\newcommand{\calN}{{\cal N}}

\newcommand{\littlesum}{\mathop{\textstyle \sum}}

\newcommand{\bb}{\mathbf{b}}
\newcommand{\be}{\mathbf{e}}

\newcommand{\bA}{\mathbf{A}}
\newcommand{\bB}{\mathbf{B}}

\newcommand{\bH}{\mathbf{H}}

\newcommand{\orthor}{\mathbf{O}}

\newcommand{\gaus}{\mathcal{N}}

\author{
Ilias Diakonikolas\thanks{Supported by NSF Medium Award CCF-2107079,
NSF Award CCF-1652862 (CAREER), a Sloan Research Fellowship, and
a DARPA Learning with Less Labels (LwLL) grant.}\\
University of Wisconsin-Madison\\
{\tt ilias@cs.wisc.edu}\\
\and
Daniel M. Kane\thanks{Supported by NSF Medium Award CCF-2107547, and
NSF Award CCF-1553288 (CAREER).}\\
University of California, San Diego\\
{\tt dakane@cs.ucsd.edu}\\
\and
Lisheng Ren
\thanks{Supported by NSF Award CCF-1652862 (CAREER).}\\
University of Wisconsin-Madison\\
{\tt lren29@wisc.edu}\\
\and
Yuxin Sun \thanks{Supported by NSF Award CCF-1652862 (CAREER).}\\
University of Wisconsin-Madison\\
{\tt yxsun@cs.wisc.edu}\\
}
\title{SQ Lower Bounds for Non-Gaussian Component Analysis with Weaker Assumptions}

\begin{document}

\maketitle

\begin{abstract}
We study the complexity of Non-Gaussian Component Analysis (NGCA) in 
the Statistical Query (SQ) model. Prior work developed a general methodology
to prove SQ lower bounds for this task that have been applicable to a wide range 
of contexts. In particular, it was known that for any univariate 
distribution $A$ satisfying certain conditions, distinguishing between 
a standard multivariate Gaussian and a distribution that behaves like 
$A$ in a random hidden direction and like a standard Gaussian in the orthogonal 
complement, is SQ-hard. The required conditions were that (1) $A$ 
matches many low-order moments with the standard univariate Gaussian, and 
(2) the chi-squared norm of $A$ with respect to the standard Gaussian 
is finite. While the moment-matching condition is necessary for 
hardness, the chi-squared condition was only required for technical reasons.
In this work, we establish that the latter condition is indeed not necessary.
In particular, we prove near-optimal SQ lower bounds for NGCA 
under the moment-matching condition only. Our result naturally 
generalizes to the setting of a hidden subspace. 
Leveraging our general SQ lower bound, we obtain near-optimal SQ lower bounds 
for a range of concrete estimation tasks 
where existing techniques provide sub-optimal or even vacuous guarantees.
\end{abstract}

\setcounter{page}{0}
\thispagestyle{empty}

\newpage

\section{Introduction} \label{sec:intro}

Non-Gaussian Component Analysis (NGCA) is a statistical estimation task 
first considered in the signal processing literature~\cite{JMLR:blanchard06a}.
As the name suggests, the objective is to find a non-gaussian direction 
(or, more generally, low-dimensional subspace) in a high-dimensional dataset.
Since its introduction, the NGCA problem has been studied
in a range of works from an algorithmic standpoint; 
see~\cite{TanV18, GoyalS19} and references therein. Here we explore this
problem from a hardness perspective in the Statistical Query (SQ) model. 
Before we motivate and state our results, we require basic background
on the SQ model. 

\paragraph{SQ Model}
SQ algorithms are a class of algorithms that are allowed
to query expectations of bounded functions on the underlying distribution 
through an (SQ) oracle
rather than directly access
samples. The model was introduced by Kearns~\cite{Kearns:98} as a natural restriction of the PAC 
model~\cite{Valiant:84} in the context of learning Boolean functions. 
Since then, the SQ model has been extensively studied in a range of settings, including
unsupervised learning~\cite{Feldman16b}.
The class of SQ algorithms is broad and captures a range of known 
algorithmic techniques in machine learning including spectral techniques,
moment and tensor methods, local search (e.g., EM),
and many others (see, e.g.,~\cite{FeldmanGRVX17, FeldmanGV17} 
and references therein).

\begin{definition}[SQ Model] \label{def:sq}
Let $D$ be a distribution on $\R^n$. 
A \emph{statistical query} is a bounded function $q:\R^n\rightarrow[-1,1]$. 
We define $\mathrm{STAT}(\tau)$ to be the oracle that given any such query $q$, outputs a value $v$ such that $|v-\E_{\bx\sim D}[q(\bx)]|\leq\tau$, where $\tau>0$ is the \emph{tolerance} parameter of the query.
A \emph{statistical query (SQ) algorithm} is an algorithm 
whose objective is to learn some information about an unknown 
distribution $D$ by making adaptive calls to the corresponding $\mathrm{STAT}(\tau)$ oracle.
\end{definition}

The following family of high-dimensional distributions forms the basis 
for the definition of the NGCA problem.

\begin{definition}[Hidden-Subspace Distribution]\label{def:hd} 
For a distribution $A$ supported on $\R^m$ and a matrix $\bV\in\R^{n\times m}$ 
with $\bV^\intercal\bV=\bI_m$, we define the distribution $\p_\bV^A$ supported on $\R^n$ 
such that it is distributed according to $A$ in the subspace $\mathrm{span}(\bv_1,\ldots,\bv_m$) and is an independent standard Gaussian in the orthogonal directions, where $\bv_1,\ldots,\bv_m$ denote the column vectors of $\bV$.
In particular, if $A$ is a continuous distribution with probability density function $A(\by)$, 
then $\p_\bV^A$ is the distribution over $\R^n$ with probability density function 
\[ \p_\bV^A(\bx)=A(\langle\bv_1,\bx\rangle,\ldots,\langle\bv_m,\bx\rangle)\exp(-\|\bx-\bV\bV^\intercal\bx\|_2^2/2)/(2\pi)^{(n-m)/2} \;.\]
\end{definition}

That is, $\p^{A}_{\bV}$ is the product distribution whose orthogonal 
projection onto the subspace $\bV$ is $A$,
and onto the subspace perpendicular to $\bV$ is 
the standard $(n-m)$-dimensional normal distribution.
An important special case of the above definition corresponds to $m=1$ 
(i.e., the case when the hidden subspace is one-dimensional); 
for this setting, we will use the notation $\p^{A}_{\bv}$ for such a distribution, 
where $A$ is a one-dimensional distribution and $\bv\in \R^n$ is a unit vector.

Since we are focusing on establishing hardness, 
we will consider the following hypothesis testing version of NGCA 
(since the learning/search version typically reduces to the testing problem).
We use $\gaus_n$ to denote the standard $n$-dimensional Gaussian distribution $\gaus(0,\bI_n)$.
We use $U(\orthor_{n,m})$ to denote the uniform distribution over the set of all orthogonal matrices 
$\bV\in \R^{n\times m}$; namely, this is the distribution obtained 
by taking $\mathbf{R}\bV'$, where 
$\mathbf{R}\in \R^{n\times n}$ is a random rotation matrix and
$\bV'\in \R^{n\times m}$ is an arbitrary orthogonal matrix.

\begin{definition}[Hypothesis Testing Version of NGCA]\label{def:hyp-test-NGCA-high}
Let $n>m\ge1$ be integers. For a distribution $A$ supported on $\R^m$,
one is given access to a distribution $D$ such that either:
\begin{itemize}[leftmargin=*]
\item $H_0$: $D=\gaus_n$,
\item $H_1$: $D$ is given by $\p_\bV^A$,
where $\bV\sim U(\orthor_{n,m})$.
\end{itemize}
The goal is to distinguish between these two cases $H_0$ and $H_1$.
\end{definition}

For the special case that $m=1$ (i.e., for a univariate distribution $A$), 
prior work~\cite{DKS17-sq} established SQ-hardness of NGCA\footnote{While 
the SQ lower bound result of ~\cite{DKS17-sq} was phrased for the search version of NGCA, 
it can be directly translated to the testing version; see, e.g., Chapter~8 of~\cite{DK}.} under the following condition: 

\begin{cond} \label{cond:moments}
Let $d \in \Z_+$. The distribution $A$ on $\R$ is such that
(i) the first $d$ moments of $A$ agree with the first $d$ moments of $\calN(0,1)$,
and (ii) the chi-squared distance $\chi^2(A, \gaus)$ is finite, where the chi-squared distance of two distributions (with probability density functions) $P, Q:\R^n\to\R_+$ is defined as
$\chi^2(P,Q)\eqdef \int_{\bx\in \R^n} P(\bx)^2/Q(\bx) d\bx -1$.
\end{cond}

\noindent Specifically, the main result of \cite{DKS17-sq} shows that any SQ algorithm that 
solves the testing version of NGCA requires either $2^{n^{\Omega(1)}}$ many SQ 
queries or at least one query with accuracy
\[n^{-\Omega(d)} \sqrt{\chi^2(A,\calN(0,1))}.\]
It is worth noting that subsequent works
(see~\cite{DK20-Massart-hard} and~\cite{DiakonikolasKPZ21}) 
generalized this result so that it only requires that
(i) $A$ {\em approximately} matches moments with the standard Gaussian,
and (ii) $A$ is a low-dimensional distribution embedded in a hidden low-dimensional subspace, 
instead of a one-dimensional distribution.

The starting point of our investigation is a key technical limitation of this line of work.
Specifically, if $\chi^2(A,\calN(0,1))$ is very large (or infinite), e.g., if $A$ has 
constant probability mass on a discrete set, the aforementioned 
SQ lower bound of ~\cite{DKS17-sq} can be very weak (or even vacuous). It is thus 
natural to ask if the finite chi-squared assumption is in fact necessary 
for the corresponding SQ lower bounds to hold.

A concrete motivation to answer this question 
comes from the applications of a generic SQ lower bound
for NGCA to various learning problems.
The SQ-hardness of NGCA can be used to obtain similar hardness for a 
number of well-studied learning problems that superficially appear very different. 
These include learning mixture models~\cite{DKS17-sq, DiakonikolasKPZ23, DKS23}, 
robust mean/covariance estimation~\cite{DKS17-sq}, robust linear regression~\cite{DKS19}, 
learning halfspaces and other natural concepts 
with adversarial or semi-random label noise~\cite{DKZ20,  GoelGK20, DK20-Massart-hard, DiakonikolasKPZ21, DKKTZ21-benign}, 
list-decodable mean estimation and linear regression~\cite{DKS18-list, DKPPS21}, 
learning simple neural networks~\cite{DiakonikolasKKZ20}, 
and even learning simple generative models~\cite{Chen0L22}. 
In several of these applications, the requirement of bounded 
chi-squared distance is somewhat problematic, in some cases leading 
to quantitatively sub-optimal results. Moreover, in certain 
applications, this restriction leads to vacuous guarantees.

\subsection{Our Results} \label{sec:results}

\subsubsection{Main Result} 

Our main result is a near-optimal 
SQ lower bound for the NGCA problem,
assuming only the moment-matching condition (i.e., without the chi-squared distance restriction).
Informally, we essentially show that in order to solve the NGCA in $n$ dimensions 
with an $m$-dimensional distribution $A$ that (approximately) matches moments 
with the standard Gaussian up to degree $d$, 
any SQ algorithm will either require one query with accuracy 
$O_{m,d}(n^{-\Omega(d)})$ or exponential in $n$ many queries.

Formally, we establish the following theorem.

\begin{theorem}[Main SQ Lower Bound Result]\label{thm:main-result}
Let $\lambda\in (0,1)$ and $n,m,d\in \mathbb{N}$ with $d$ even and $m,d\leq n^{\lambda}/\log n$.
Let $0<\nu<2$ and $A$ be a distribution on $\R^m$ such that
for any polynomial $f:\R^m\to \R$ of degree at most $d$ and $\E_{\bx\sim\gaus_m}[f(\bx)^2]=1$, the following holds: 
$$|\E_{\bx\sim A}[f(\bx)]-\E_{\bx\sim \gaus_m}[f(\bx)]|\leq \nu \;.$$
Let $0<c<(1-\lambda)/4$ and $n$ be at least a sufficiently large constant depending on $c$.  
Then any SQ algorithm solving the $n$-dimensional NGCA problem 
(as in Definition~\ref{def:hyp-test-NGCA-high})
with $2/3$ success probability requires either 
\begin{itemize}
\item[(i)] a query of tolerance $O_{m,d}\left (n^{-((1-\lambda)/4-c) d}\right )+(1+o(1))\nu$,
or 
\item[(ii)] $2^{n^{\Omega(c)}}$ many queries. 
\end{itemize}
\end{theorem}

A few comments regarding the parameters are in order here.
We first note that the constant in $O_{m,d}(n^{-((1-\lambda)/4-c) d})$ is roughly the size of the Binomial coefficient 
$\binom{(d+m)/2-1}{m/2-1}$. In addition, we would like to point out that for most applications we will have $m,d \ll \poly(n)$.
Therefore, the parameters $\lambda, c$ can essentially be taken arbitrarily close to $0$.
So, informally, our lower bound on the query accuracy can be roughly thought 
of as $\binom{ (d+m)/2-1}{m/2-1} n^{-d/4}$.
It is worth noting that the constant of $1/4$ appearing in the exponent is optimal in the worst-case; 
there is a problem instance of NGCA that can be solved 
by an SQ algorithm with a single query of tolerance $n^{-(d+2)/4}$ (see Appendix~\ref{app:optimality-proof}).

For all the applications given in this paper, the above theorem will be applied 
for the special case that $m=1$; namely, the case that the hidden subspace 
is a hidden direction and $A$ is a univariate distribution.

\paragraph{Relation to LLL-based Algorithms for NGCA} 
Consider the special case of the NGCA problem corresponding to $m=1$, where $A$ is a {\em discrete} distribution that matches its first $d$ moments with the standard Gaussian. Theorem~\ref{thm:main-result} implies that any SQ algorithm for this version of the problem either uses a query 
with accuracy $n^{-\Omega(d)}$ or exponential many queries. 
On the other hand, recent works~\cite{DK22LLL, ZSWB22LLL} 
gave polynomial-time algorithms for this problem with sample complexity $O(n)$, regardless of the degree $d$.
It is worth noting that the existence of these algorithms 
does not contradict our SQ lower bound, as these algorithms
are based on the LLL-method for lattice basis-reduction 
that is not captured by the SQ framework. As an implication, 
it follows that LLL-based methods surpass any efficient 
SQ algorithm for these settings of NGCA. A similar observation was previously made in \cite{DH23} for the special case that the discrete distribution $A$ is supported on $\{0,\pm 1\}$; consequently, \cite{DH23} could only obtain a quadratic separation. Finally, we note that this limitation of SQ algorithms is also shared by two other prominent restricted families of algorithms (namely, SoS algorithms and low-degree polynomial tests).

\subsubsection{Applications} 
We believe that Theorem~\ref{thm:main-result} is interesting in its own 
right, as it elucidates the SQ complexity of a well-studied statistical problem. 
Here we discuss concrete applications to some  
natural statistical tasks.
We note that the SQ lower bound in the 
prior work~\cite{DKS17-sq} cannot give optimal (or even nontrivial) lower bounds for these applications.

\paragraph{List-decodable Gaussian Mean Estimation}
One can leverage our result to obtain a sharper SQ lower bound for 
the task of list-decodable Gaussian mean estimation.
In this task, the algorithm is given as input points from $\R^n$ 
where an $\alpha<1/2$ fraction of the points are drawn from 
an unknown mean and identity covariance Gaussian 
$\mathcal{N}(\bm{\mu},\bI)$, and the remaining points are arbitrary.
The goal of the algorithm is to output a list of $O(1/\alpha)$ 
many hypothesis vectors at least one of which 
is close to $\bm{\mu}$ in $\ell_2$-norm
with probability at least $2/3$.
\cite{DKS18-list} established the following SQ lower bound for this problem 
(see also~\cite{DK} for a different exposition).

\begin{fact} [\cite{DKS18-list}] \label{fct:previous-decodable-bound}
For each $d\in \Z_+$ and $c\in (0,1/2)$, there exists $c_d>0$ such that 
for any $\alpha>0$ sufficiently small the following holds.
Any SQ algorithm that is given access to a
$(1-\alpha)$-corrupted Gaussian $\gaus(\bm{\mu},\bI)$ 
in $n>d^{3/c}$ dimensions 
and returns a list of hypotheses such that with probability at least $2/3$ 
one of the hypotheses is within $\ell_2$-distance $c_d\alpha^{-1/d}$ 
of the true mean $\bm{\mu}$, 
does one of the following: 
(i)  Uses queries with error tolerance at most 
$\exp (O(\alpha^{-2/d}))\Omega(n)^{-(d+1)(1/4-c/2)}$. 
(ii) Uses at least $\exp(\Omega(n^c))$ many queries. 
(iii) Returns a list of at least $\exp(\Omega(n))$ many hypotheses.
\end{fact}

The above statement is obtained using the framework of \cite{DKS17-sq} 
by considering the distribution testing problem 
between $\gaus(\bm{\mu}, \bI)$ and $\p_\bv^A$,
where $A$ is the one-dimensional moment-matching distribution 
of the following form. 

\begin{fact} [\cite{DKS18-list}, see Lemma 8.21 in \cite{DK}]\label{fct:decodable-moment-matching-distribution}
    For $d\in \Z_+$, there exists a distribution 
    $A=\alpha\gaus(\mu,1)+(1-\alpha)E$,
    for some distribution $E$ and $\mu=10c_d\alpha^{-1/d}$, 
    such that the first $d$ moments of $A$ agree with those of $\gaus(0,1)$.
    Furthermore, the probability density function of $E$ can be taken to be pointwise 
    at most twice the pdf of the standard Gaussian.
\end{fact}

We note that $\chi^2(A, \gaus)=O(\exp(\mu^2))$, 
and this results in the $\exp (O(\alpha^{-2/d}))$ term in the error tolerance 
of Fact \ref{fct:previous-decodable-bound}.
Consequently, if $\alpha\ll \log(n)^{-d/2}$, 
the error tolerance is greater than one 
and Fact \ref{fct:previous-decodable-bound} fails 
to give a non-trivial bound.
It is worth noting that the setting of ``small $\alpha$'' (e.g., $\alpha$ is sub-constant in the dimension) is of significant interest in various applications, including in mean estimation. 
A concrete application is in the related crowdsourcing setting of~\cite{MeisterV18} dealing with this parameter regime.

We can circumvent this technical problem 
by combining our main result (Theorem \ref{thm:main-result})
with Fact \ref{fct:decodable-moment-matching-distribution} 
to obtain the following sharper SQ lower bound 
(see Appendix \ref{app:omit-application-proof} for the proof).

\begin{theorem}[SQ Lower Bound for List-Decoding the Mean] \label{thm:decodable-bound}
Let $\lambda\in (0,1)$ and $n,d\in \mathbb{N}$ with $d$ be even and $d\leq n^{\lambda}/\log n$.
Let $c>0$ and $n$ be at least a sufficiently large constant depending on $c$.
There exists $c_d>0$ such that 
for any $\alpha>0$ sufficiently small the following holds.
Any SQ algorithm that is given access to a $(1-\alpha)$-corrupted Gaussian $\gaus(\bm{\mu},\bI)$ 
in $n$ dimensions 
and returns a list of hypotheses such that with probability at least $2/3$ 
one of the hypotheses is within $\ell_2$-distance $c_d\alpha^{-1/d}$ 
of the true mean $\bm{\mu}$, 
does one of the following: 
(i) Uses queries with error tolerance at most $O_d(n^{-((1-\lambda)/4-c) d})$.
(ii) Uses at least $2^{n^{\Omega(c)}}$ many queries.
(iii) Returns a list of at least $\exp(\Omega(n))$ many hypotheses. 
\end{theorem}

\paragraph{Anti-concentration Detection}
Anti-concentration (AC) detection is the following hypothesis testing problem:
given access to an unknown distribution $D$ over $\R^n$,
where the input distribution $D$ is promised to satisfy either 
(i) $D$ is the standard Gaussian; or
(ii) $D$ has at least $\alpha<1$ of its probability mass residing
inside a dimension $(n-1)$ subspace $V\subset \R^n$.
The goal is to distinguish between the two cases 
with success probability $2/3$.

In order to use our main result to derive an SQ lower bound for this task, 
we require the following lemma on univariate moment-matching
(see Appendix \ref{app:omit-application-proof} for the proof).

\begin{lemma} \label{lem:anti-concentration-detection-moment-matching}
    Let $D_0$ denote the distribution that outputs $0$ with probability $1$.
    For $d\in \Z_+$, there exists $\alpha_d \in (0, 1)$ and     
    a univariate distribution     
    $A=\alpha_d D_0+(1-\alpha_d)E$
    for some distribution $E$, 
    such that the first $d$ moments of $A$ agree with those of $\gaus(0,1)$.
    Furthermore, the pdf of $E$ can be taken to be pointwise 
    at most twice the pdf of the standard Gaussian.
\end{lemma}

Using the above lemma and our main result, we can deduce the following 
SQ lower bound on the anti-concentration detection problem
(see Appendix \ref{app:omit-application-proof} for the proof).

\begin{theorem}[SQ Lower Bound for AC Detection] \label{thm:anti-concentration-detection}
    For any $\alpha\in (0,1/2)$, any SQ algorithm 
    that has access to an $n$-dimensional distribution that is either
    (i) the standard Gaussian; or
    (ii) a distribution that has at least $\alpha<1$ probability mass 
    in an (n-1)-dimensional subspace $V\subset \R^n$,
    and distinguishes the two cases with success probability at least $2/3$,
    either requires a query with error at most $n^{-\omega_\alpha(1)}$,
    or uses at least $2^{n^{\Omega(1)}}$ many queries.
\end{theorem}

\paragraph{Learning Periodic Functions}
Another application is for the well-studied problem of learning {\em periodic} functions (see, e.g.,~\cite{SVWX17} and~\cite{SZB21}), which is closely related to the continuous Learning with Errors (cLWE) problem~\cite{BRST21}. 
In the task of learning periodic functions, the algorithm is 
given sample access to a distribution $D$ of labeled 
examples $(\bx,y)$ over $\R^n\times \R$.
The distribution $D$ is such that $\bx\sim \gaus(0,\bI_n)$ and $y=\cos(2\pi(\delta\langle \bw,\bx\rangle +\zeta))$ with noise $\zeta\sim \gaus(0,\sigma^2)$.
This implies that $y$ is a periodic function of $\bx$ along the direction of $\bw$ with a small amount of noise. 
While 
the frequency and noise scale parameters $\delta,\sigma\in \R_+$ 
are known to the algorithm,
the parameter $\bw\in \mathbb{S}^{n-1}$ 
is unknown; the goal of the algorithm is to output a hypothesis $h:\R^n\to \R$
such that $\E_{\bx\sim D}[(h(\bx)-y)^2]$ is minimized.

In order to use our main theorem to derive an SQ lower bound for this problem,
we need to show that for any $t\in [-1,1]$, $D$ conditioned on $y=t$ approximately matches moments with 
$\gaus(0,\bI_n)$. To this end, we introduce the following definition and fact of discrete Gaussian measure from \cite{DK20-Massart-hard}.

\begin{definition}
    For $s\in \R_+$ and $\theta\in \R$, let $G_{s,\theta}$ denote the measure of the 
    ``$s$-spaced discrete Gaussian distribution''. In particular, for each $n\in \Z$, $G_{s,\theta}$
    assigns mass $s g(ns+\theta)$ to the point $ns+\theta$, where $g$ is the pdf function of $\gaus(0,1)$.
\end{definition}

Note that although $G_{s,\theta}$ is not a probability measure (as the total measure is not one),
it can be thought of as a probability distribution since the total measure is close to one for small $\sigma$.
To see this, we introduce the following fact.

\begin{fact} [Lemma 3.12 from \cite{DK20-Massart-hard}] \label{fct:discrete-gaussian-moment}
    For all $k\in \N$, $s>0$ and all $\theta\in \R$, we have that 
    $|\E_{t\sim \gaus(0,1)}[t^k]-\E_{t\sim G_{s,\theta}}[t^k]|=k!O(s)^k\exp(-\Omega(1/s^2))$.
\end{fact}

Using the above fact, we are now ready to prove our SQ lower bound for learning periodic functions (see Appendix \ref{app:omit-application-proof} for the proof).

\begin{theorem} [SQ Lower Bound for Learning Periodic Functions]\label{thm:SQ-periodic}
    Let $c>0$ and $n$ be at least a sufficiently large constant depending on $c$. Let $D$ be the distribution of $(\bx,y)$ over $\R^n\times \R$ that is generated by 
    $\bx\sim \gaus(0,\bI_n)$ and $y=\cos(2\pi(\delta\langle \bw,\bx\rangle +\zeta))$ with noise $\zeta\sim \gaus(0,\sigma^2)$.
    Let $\delta>n^{c}$, $\sigma$ be known and $\bw$ be unknown to the algorithm.   
    Then any SQ algorithm that has access to the distribution $D$ and returns a hypothesis $h:\R^n\to \R$ such that $\E_{(\bx,y)\sim D}[(h(\bx)-y)^2]=o(1)$, either requires a query with error at most $\exp(-n^{c'})$ for $c'<\min(2c,1/10)$, or uses at least $2^{n^{\Omega(1)}}$ many queries.
\end{theorem}

\subsection{Technical Overview} \label{sec:techniques}
We start by noting that we cannot use the standard SQ dimension argument~\cite{FeldmanGRVX17}
to prove our result, due to the unbounded chi-squared norm.
To handle this issue, we need to revisit the underlying ideas of 
that proof. In particular, we will show that for any bounded query function 
$f:\R^n\to[-1,1]$, with high probability over 
some $\bV\sim U(\orthor_{n,m})$, the quantity 
$|\E_{\bx\sim\gaus_n}[f(\bx)] - \E_{\bx\sim\p_\bV^A}[f(\bx)]|$ 
will be small.
This allows an adversarial oracle to return $\E_{\bx\sim\gaus_n}[f(\bx)]$ 
to every query $f$ regardless of which case we are in,
unless the algorithm is lucky enough (or uses high accuracy) 
to find an $f$ that causes 
$|\E_{\bx\sim\gaus_n}[f(\bx)] - \E_{\bx\sim\p_\bV^A}[f(\bx)]|$ 
to be at least $\tau$.

Our approach for calculating $\E_{\bx\sim\p_\bV^A}[f(\bx)]$ 
will be via Fourier analysis.
In particular, using the Fourier decomposition of $f$, 
we can write 
$f(\bx) = \sum_{k=0}^\infty \langle \bT_k ,\bH_k(\bx)\rangle$,
where $\bH_k(\bx)$ is the properly normalized degree-$k$ Hermite polynomial tensor and $\bT_k$ is the degree-$k$ Fourier coefficient of $f$.
Taking the inner product with the distribution $\p_\bV^A$ involves computing the expectation of $\E_{\bx\sim\p_\bV^A}[\bH_k(\bx)]$,
which can be seen to be $ \bV^{\otimes k}\bA_k$,
where $\bA_k=\E_{\bx\sim A} [\bH_k(\bx)]$.
Thus, we obtain (at least morally speaking) that
\begin{align}\label{eq:Hermite-expansion}
\E_{\bx\sim\p_\bV^A}[f(\bx)] = \sum_{k=0}^\infty \langle \bV^{\otimes k}\bA_k,\bT_k\rangle\le\sum_{k=0}^\infty|\langle \bA_k,(\bV^{\intercal})^{\otimes k}\bT_k\rangle |\le\sum_{k=0}^\infty\|\bA_k\|_2\|(\bV^{\intercal})^{\otimes k}\bT_k\|_2 \;.
\end{align}
The $k=0$ term of the above sum will be exactly 
$\bT_0 = \E_{\bx\sim\gaus_n}[f(\bx)]$.
By the moment-matching condition, the $k=1$ through $k=d$ terms will be very small. Finally, we need to argue that 
the higher degree terms are small with high probability. 
To achieve this, we provide a non-trivial upper bound for $\E_{\bV\sim U(\orthor_{n,m})}[\|(\bV^\intercal)^{\otimes k}\bT_k\|_2^2]$.
Combining with the fact that $\sum_{k=0}^{\infty} \|\bT_k\|_2^2 = \|f\|_2^2 \le 1$,
this allows us to prove high probability bounds on each term, 
and thus their sum.

Furthermore, if we want higher probability estimates of the terms 
with small $k$, we can instead bound 
$\E_{\bV\sim U(\orthor_{n,m})}[\|(\bV^\intercal)^{\otimes k}\bT_k\|_2^{2a}]$ for some integer $a$.
Unfortunately, there are non-trivial technical issues with the above approach, 
arising from issues with~\eqref{eq:Hermite-expansion}.
To begin with, as no assumptions other than moment-matching 
(for a few low-degree moments) were made on $A$,
it is not guaranteed that $\bA_k$ is finite 
for larger values of $k$.
To address this issue, we will truncate the distribution $A$.
In particular, we pick a parameter $B$ (which will be determined carefully),
and define $A'$ to be $A$ conditioned on the value in the ball $\mathbb{B}^m(B)$.
Due to the higher-moment bounds, we can show that $A$ and $A'$ 
are close in total variation distance,  
and thus that $\E_{\bx\sim \p_\bV^A}[f(\bx)]$ is close to $\E_{\bx\sim\p_\bV^{A'}}[f(\bx)]$ 
for any bounded $f$.

Furthermore, using the higher moment bounds, 
we can show that $A'$ {\em nearly} matches the low-degree moments 
of $\gaus_m$. The second issue arises with the interchange of summations 
used to derive~\eqref{eq:Hermite-expansion}.
In particular, although 
$f(\bx) = \sum_{k=0}^\infty \langle \bT_k,\bH_k(\bx)\rangle$, 
it does not necessarily follow that we can interchange the infinite sum 
on the right-hand-side with taking the expectation over $\bx\sim\p^A_\bV$.
To fix this issue, we split $f$ into two parts $f^{\le \ell}$ 
(consisting of its low-degree Fourier components) and $f^{> \ell}$.
We note that~\eqref{eq:Hermite-expansion} {\em does} hold for $f^{\leq \ell}$, 
as the summation there will be finite, and we can use the above argument 
to bound 
$\E_{\bx\sim\gaus_n}[f(\bx)] - \E_{\bx\sim\p_\bV^A}[f^{\le\ell}(\bx)]|$ with high probability.
To bound $|\E_{\bx\sim\p_\bV^A}[f^{> \ell}(\bx)]|$,
we note that by taking $\ell$ large, we can make 
$\|f^{> \ell}\|_2 < \delta$ for some exponentially small $\delta>0$.
We then bound
$\E_{\bV\sim U(\orthor_{n,m})}[\E_{\bx\sim\p_\bV^A}[|f^{>\ell}| ] ] = \E_{\bx\sim \q}[|f^{> \ell}|]$,
where $\q$ is the average over $\bV$ of $\p^A_\bV$ 
(note that everything here is non-negative, so there is no issue with the 
interchange of integrals). Thus, we can bound the desired quantity by noting that 
$\|f^{> \ell}\|_2$ is small and that the chi-squared norm of $\mathbf{Q}$ with respect to the standard Gaussian $\mathcal{N}_n$ is bounded.

\section{Preliminaries} \label{sec:prelims}

We will use lowercase boldface letters for vectors and capitalized boldface letters for matrices and tensors.
We use $\mathbb{S}^{n-1}=\{\bx\in\R^n:\|\bx\|_2=1\}$ to denote the $n$-dimensional unit sphere.
For vectors $\bu,\bv\in\R^n$, we use $\langle\bu,\bv\rangle$ to denote the standard inner product.
For $\bu\in \R^n$, we use $\|\bu\|_k=\big (\sum_{i=1}^n \bu_i^k\big )^{1/k} $ to denote the $\ell_k$-norm of $\bu$.
For tensors, 
we will consider a $k$-tensor to be an element in $(\mathbb{R}^n)^{\otimes k}\cong\mathbb{R}^{n^k}$.
This can be thought of as a vector with $n^k$ coordinates.
We will use $\bA_{i_1,\ldots,i_k}$ to denote the coordinate of a $k$-tensor $\bA$ indexed by the $k$-tuple $(i_1,\ldots,i_k)$. 
By abuse of notation, we will sometimes also use this to denote the entire tensor.
The inner product and $\ell^k$-norm of a $k$-tensor are defined by thinking of the tensor as a vector with $n^k$ coordinates
and then using the definition of inner product and $\ell_k$-norm of vectors.
For a vector $\bv\in\R^n$, we denote by $\bv^{\otimes k}$ to be a vector (linear object) in $\R^{n^k}$.
For a matrix $\bV\in\R^{n\times m}$, we denote by $\|\bV\|_2,\|\bV\|_F$ to be the operator norm and Frobenius norm respectively.
In addition, we denote by $\bV^{\otimes k}$ to be a matrix (linear operator) mapping $\R^{n^k}$ to $\R^{m^k}$.

We use $\mathds{1}$ to denote the indicator function of a set, 
specifically $\mathds{1}(t\in S)=1$ if $t\in S$ and $0$ otherwise. 
We will use $\Gamma:\R\rightarrow \R$ to denote the gamma function $\Gamma(z)=\int_0^{\infty} t^{z-1}e^{-t} dt$.
We use $B:\R\times \R\rightarrow \R$ to denote the beta function 
$B(z_1,z_2)=\Gamma(z_1)\Gamma(z_2)/{\Gamma(z_1+z_2)}$.
We use $\chi^2_k$ to denote the chi-squared distribution with $k$ degrees of freedom.
We use $\mathrm{Beta}(\alpha,\beta)$ to denote the Beta distribution 
with parameters $\alpha$ and $\beta$. 

For a distribution $D$, we use $\pr_{D}[S]$ to denote the probability of an event $S$.
For a continuous distribution $D$ over $\R^n$, we sometimes use $D$ for both the distribution itself and its probability density function.
For two distributions $D_1,D_2$ over a probability space $\Omega$,
let $\dtv(D_1,D_2)=\sup_{S\subseteq\Omega}|\pr_{D_1}(S)-\pr_{D_2}(S)|$
denote the total variation distance between $D_1$ and $D_2$.
For two continuous distribution $D_1, D_2$ over $\R^n$, 
we use $\chi^2(D_1,D_2)=\int_{\R^n}D_1(\bx)^2/D_2(\bx)d\bx-1$
to denote the chi-square norm of $D_1$ w.r.t.\;$D_2$.
For a subset $S\subseteq \R^n$ with finite measure or finite surface measure, 
we use $U(S)$ to denote the uniform distribution over $S$ (w.r.t. Lebesgue measure for the volumn/surface area of $S$).

\paragraph{Basics of Hermite Polynomials}
We require the following definitions. 
\begin{definition}[Normalized Hermite Polynomial]\label{def:Hermite-poly}
For $k\in\N$,
we define the $k$-th \emph{probabilist's} Hermite polynomials
$\mathrm{\textit{He}}_k:\R\to \R$
as
$\mathrm{\textit{He}}_k(t)=(-1)^k e^{t^2/2}\cdot\frac{d^k}{dt^k}e^{-t^2/2}$.
We define the $k$-th \emph{normalized} Hermite polynomial 
$h_k:\R\to \R$
as
$h_k(t)=\mathrm{\textit{He}}_k(t)/\sqrt{k!}$.
\end{definition}
Furthermore, we will use multivariate Hermite polynomials in the form of
Hermite tensors 
(as the entries in the Hermite tensors are 
rescaled multivariate Hermite polynomials).
We define the \emph{Hermite tensor} as follows.
\begin{definition}[Hermite Tensor]\label{def:Hermite-tensor}
For $k\in \N$ and $\bx\in\R^n$, we define the $k$-th Hermite tensor as
\[
(\bH_k(\bx))_{i_1,i_2,\ldots,i_k}=\frac{1}{\sqrt{k!}}\sum_{\substack{\text{Partitions $P$ of $[k]$}\\ \text{into sets of size 1 and 2}}}\bigotimes_{\{a,b\}\in P}(-\bI_{i_a,i_b})\bigotimes_{\{c\}\in P}\bx_{i_c}\; .
\]
\end{definition}

We denote by $L^2(\R^n,\mathcal{N}_n)$ the space of all 
functions $f:\R^n\to\R$ such that 
$\E_{\bv\sim\mathcal{N}_n}[f^2(\bv)]<\infty$.
For functions $f,g\in L^2(\R^n,\mathcal{N}_n)$, we use $\langle f,g\rangle_{\gaus_n}=\E_{\bx\sim \gaus_n} [f(\bx)g(\bx)]$ to denote their inner product.
We use $\|f\|_2=\sqrt{\langle f,f\rangle_{\gaus_n}}$ to denote its $L^2$-norm.
For a function $f:\R^n \to \R$ and $\ell\in \N$, we use $f^{\leq\ell}$ to denote 
$f^{\leq \ell}(\bx)=\sum_{k=0}^\ell \langle \bA_k, \bH_k(\bx)\rangle$,
where $\bA_k=\E_{\bx\sim \gaus_n}[f(\bx)\bH_k(\bx)]$,
which is the degree-$\ell$ approximation of $f$.
We use $f^{>\ell}=f-f^{\leq \ell}$ to denote its residue.
We remark that normalized Hermite polynomials (resp. Hermite tensors) form a complete orthogonal system for the inner product space $L^2(\R,\mathcal{N})$ 
(resp. $L^2(\R^n,\mathcal{N}_n)$).
This implies that for $f\in L^2(\R^n,\mathcal{N}_n)$,
$\lim_{\ell\rightarrow \infty} \|f^{>\ell}\|_2 = 0$.
We also remark that 
both our definition of Hermite polynomial and Hermite tensor are
``normalized'' in the following sense:
For Hermite polynomials, it holds $\|h_k\|_2=1$.
For Hermite tensors, given any symmetric tensor $A$,
we have $\|\langle\bA,\bH_k(\bx)\rangle\|_2^2=\langle\bA,\bA\rangle$.

The following claim states that 
for any orthonormal transformation $\bB$, 
the Hermite tensor $\bH_k(\bB \bx)$ 
can be written as applying the linear transformation 
$\bB^{\otimes k}$ on the Hermite tensor $\bH_k(\bx)$. 
The proof is deferred to Appendix~\ref{sec:omit-prelim}.

\begin{claim} \label{clm:othor-tran}
Let $1\le m<n$. Let $\bB\in\R^{m\times n}$ with $\bB\bB^\intercal=\bI_m$.
It holds that
$\bH_k(\bB\bx)=\bB^{\otimes k}\bH_k(\bx),\bx\in\R^n$.
\end{claim}

\section{SQ-Hardness of NGCA: Proof of Theorem~\ref{thm:main-result}} \label{sec:ngca-sq}

The main idea of the proof is the following.
Suppose that the algorithm only asks queries with tolerance $\tau$,
and let $f$ be an arbitrary query function that the algorithm selects.
The key ingredient is to show that 
$ |\E_{\bx\sim \p^A_\bV}[f(\bx)]-\E_{\bx\sim \gaus_n}[f(\bx)]|\leq \tau$ with high probability
over $\bV\sim U(\orthor_{n,m})$.
If this holds, 
then when the algorithm queries $f$,
if the input is from the alternative hypothesis,
with high probability,
$\E_{\bx\sim \gaus_n}[f(\bx)]$ is a valid answer for the query.
Therefore, when the algorithm queries $f$,
regardless of whether the input is from the alternative 
or null hypothesis, the oracle can just return 
$\E_{\bx\sim \gaus_n}[f(\bx)]$.
Then the algorithm will not observe any difference 
between the two cases with any small number of queries. 
Thus, it is impossible to distinguish the two cases 
with high probability.

To prove the desired bound, 
we introduce the following proposition.

\begin{proposition} \label{prp:main-tail-bound}
Let $\lambda\in (0,1)$ and $n,m,d\in \mathbb{N}$ with $d$ be even and $m,d\leq n^{\lambda}$.
Let $\nu\in \R_+$ and $A$ be a distribution on $\R^m$ such that
for any polynomial $f:\R^m\to \R$ of degree at most $d$ and $\E_{\bx\sim\gaus_m}[f(\bx)^2]=1$,
$$|\E_{\bx\sim A}[f(\bx)]-\E_{\bx\sim \gaus_m}[f(\bx)]|\leq \nu\; .$$
Let $0<c<(1-\lambda)/4$ and $n$ is at least a sufficiently large constant depending on $c$, 
then, for any function $f:\R^n\rightarrow[-1,1]$, it holds
\[
\pr_{\bV\sim U(\orthor_{n,m})}\left[|\E_{\bx\sim \p^{A}_{\bV}}[f(\bx)]-\E_{\bx\sim\gaus_n}[f(\bx)]|\geq
\left (\frac{\Gamma(d/2+m/2)}{\Gamma(m/2)}\right )n^{-((1-\lambda)/4-c) d}+(1+o(1))\nu\right]
\leq 2^{-n^{\Omega(c)}}\;.
\]
\end{proposition}

\noindent Assuming Proposition \ref{prp:main-tail-bound},
the proof of our main theorem is quite simple.

\begin{proof} [Proof for Theorem \ref{thm:main-result}]
Suppose there is an SQ algorithm ${\cal A}$ using $q<2^{n^{\Omega(c)}}$ many queries of accuracy 
$\tau \geq \left (\frac{\Gamma(d/2+m/2)}{\Gamma(m/2)}\right )n^{-((1-\lambda)/4-c) d}+(1+o(1))\nu$ and succeeds with at least $2/3$ probability. 
We prove by contradiction that such an ${\cal A}$ cannot exist.
Suppose the input distribution is $\gaus_n$,
and the SQ oracle always answers
$\E_{\bx\sim \gaus_n} [f(\bx)]$ for any query $f$.
Then the assumption on ${\cal A}$ implies
that it answers ``null hypothesis'' with probability $\alpha>2/3$.
Now consider the case that the input distribution is $\p_\bV^A$ and $\bV\sim U(\orthor_{n,m})$.
Suppose the SQ oracle still always answers
$\E_{\bx\sim \gaus_n} [f(\bx)]$ for any query $f$.
Let $f_1,\cdots,f_q$ be the queries the algorithm asks,
where $q = 2^{n^{\Omega(c)}}$.
By Proposition \ref{prp:main-tail-bound} and a union bound, we have
\[\pr_{\bV\sim U(\orthor_{n,m})} [\exists i\in [q],\;|\E_{\bx\sim\p^A_\bV}[f_i(\bx)]-\E_{\bx\sim \gaus_n}[f_i(\bx)] |\geq \tau]=o(1)\;.\]
Therefore, with probability $1-o(1)$, the answers given by the oracle are valid.
From our assumption on ${\cal A}$, the algorithm needs to answer ``alternative hypothesis''
with probability at least $\frac{2}{3}(1-o(1))$.
But since the oracle always answers $\E_{\bx\sim \gaus_n} [f(\bx)]$
(which is the same as the above discussed null hypothesis case), 
we know the algorithm will return ``null hypothesis'' with probability $\alpha>2/3$. 
This gives a contradiction and completes the proof.
\end{proof}

\noindent The rest of this section is devoted to the proof of Proposition~\ref{prp:main-tail-bound}.

\subsection{Fourier Analysis using Hermite Polynomials}\label{ssec:Fourier}

The main idea of Proposition \ref{prp:main-tail-bound} is to analyze $\E_{\bx\sim \p^A_\bV}[f(\bx)]$
through Fourier analysis using Hermite polynomials.
Before we do the analysis, 
we will first truncate the distribution $A$ inside $\mathbb{B}^{m}(B)$
which is the $\ell_2$-norm unit ball in $m$-dimensions with radius $B$
for some $B\in \R_+$ to be specified.
Namely, we will consider the truncated distribution $A'$ defined as the distribution of $\bx\sim A$ conditioned on $\bx\in \mathbb{B}^{m}(B)$.
The following lemma shows that given any $m$-dimensional distribution $A$ that approximately matches the first $d$ moment tensor with the Gaussian,
the truncated distribution $A'$ is close to $A$ in both
the total variation distance and the first $d$ moment tensors.

\begin{lemma} \label{lem:truncation}
Let $m,d\in\mathbb{N}$ with $d$ be even.
Let $A$ be a distribution on $\R^m$ such that
for any polynomial $f$ of degree at most $d$ and $\E_{\bx\sim\gaus_m}[f(\bx)^2]=1$,
\[
|\E_{\bx\sim A}[f(\bx)]-\E_{\bx\sim \gaus_m}[f(\bx)]|\leq \nu\leq 2\; .
\]
Let $B\in \R_+$ such that $B^d\geq c_1\left (2^{d/2}\sqrt{\frac{\Gamma(d+m/2)}{\Gamma(m/2)}}\right )$
where $c_1$ is at least a sufficiently large universal constant
and let $A'$ be the truncated distribution defined as 
the distribution of $\bx\sim A$ conditioned on $\bx\in \mathbb{B}^{m}(B)$.
Then 
$d_{\mathrm{TV}}(A,A')= O\left (2^{d/2}\sqrt{\frac{\Gamma(d+m/2)}{\Gamma(m/2)}}\right )B^{-d}$,
Furthermore, for any $k\in \N$, 
\[
\|\E_{\bx \sim A'}[\bH_k(\bx)]-\E_{\bx\sim \gaus_m}[\bH_k(\bx)] \|_2 = \begin{cases}
  2^{O(k)}\left (2^{d/2}\sqrt{\frac{\Gamma(d+m/2)}{\Gamma(m/2)}}\right )B^{-(d-k)}  \\
  +\left (1+O\left (2^{d/2}\sqrt{\frac{\Gamma(d+m/2)}{\Gamma(m/2)}}\right )B^{-d}\right )\nu  & k<d\; ; \\
  2^{O(k)}\left (2^{d/2}\sqrt{\frac{\Gamma(d+m/2)}{\Gamma(m/2)}}\right ) B^{k-d} & k\geq d\; .
\end{cases}
\]
\end{lemma}

\vspace{-0.1cm}

\noindent The proof of Lemma~\ref{lem:truncation} is deferred to Appendix~\ref{sec:omit-Fourier}.

\noindent Since $d_{\mathrm{TV}}(A,A')\leq O\left (2^{d/2}\sqrt{\frac{\Gamma(d+m/2)}{\Gamma(m/2)}}\right )B^{-d}$ 
and $f$ is bounded in $[-1,1]$, it follows that
\[
|\E_{\bx\sim \p^{A'}_\bV}[f(\bx)]-\E_{\bx\sim \p^A_\bV}[f(\bx)]|
\leq 2 d_{\mathrm{TV}}(\p^{A'}_\bV,\p^{A'}_\bV)
=2 d_{\mathrm{TV}}(A,A')=O\left(\left (2^{d/2}\sqrt{\frac{\Gamma(d+m/2)}{\Gamma(m/2)}}\right )B^{-d}\right )\; .
\]
Therefore, 
if suffices for us to 
analyze $\E_{\bx\sim \p^{A'}_\bV}[f(\bx)]$
instead of $\E_{\bx\sim \p^A_\bV}[f(\bx)]$.
Furthermore, the property that $A'$ is bounded inside $\mathbb{B}^{m}(B)$
will be convenient in the Fourier analysis later.
We introduce the following lemma which decomposes
$\E_{\bx\sim \p^{A'}_\bV}[f(\bx)]$ using Hermite analysis.

\begin{lemma}[Fourier Decomposition Lemma] \label{lem:hermite-decomposition}
Let $A'$ be any distribution supported on $\R^m$, $\bV\in \R^{n\times m}$ and $\bV^\intercal\bV=\bI_m$.
Then for any $\ell\in \N$, $\E_{\bx\sim \p^{A'}_{\bV}} [f(\bx)]=\sum_{k=0}^\ell \langle \bV^{\otimes k}\bA_k,\bT_k\rangle+\E_{\bx\sim \p^{A'}_{\bV}} [f^{>\ell}(\bx)]$,
where $\bA_k=\E_{\bx\sim A'} [\bH_k(\bx)]$ and $\bT_k=\E_{\bx\sim \gaus_n} [f(\bx) \bH_k(\bx)]$. 
\end{lemma}

\begin{proof}
Noting that $\E_{\bx\sim \p^{A'}_\bV}[f(\bx)]=\E_{\bx\sim \p^{A'}_\bV}[f^{\leq \ell}(\bx)]
+\E_{\bx\sim \p^{A'}_\bV}[f^{>\ell}(\bx)]$,
we will show that
\[\E_{\bx\sim \p^{A'}_\bV} [f^{\leq \ell}(\bx) ]
=\Big \langle \sum\nolimits_{k=0}^\ell\langle\bA_k,\bH_k(\langle\bv_1,\bx\rangle,\ldots,\langle\bv_m,\bx\rangle)\rangle, f^{\leq \ell}\Big \rangle_{\gaus_n}\; .\]
Notice that $f^{\leq \ell}$ is a polynomial of degree-$\ell$. 
Let $\bb_1,\cdots, \bb_n$ be an orthonormal basis, 
where $\bb_1=\bv_1,\ldots,\bb_m=\bv_m$.
Then degree-$\ell$ polynomials are spanned by polynomials
of the form 
$\prod_{i=1}^n \langle \bb_i, \bx\rangle^{q_i}$, 
where $\sum_{i=1}^n q_i \leq \ell$.
Therefore, it suffices to show that for any $\sum_{i=1}^n q_i\le\ell$,
\[\E_{\bx\sim \p^{A'}_\bV}\Big[\prod\nolimits_{i=1}^n \langle \bb_i, \bx\rangle^{q_i}\Big ]
=\Big \langle \sum\nolimits_{k=0}^\ell
\langle\bA_k,\bH_k(\langle\bv_1,\bx\rangle,\ldots,\langle\bv_m,\bx\rangle)\rangle,\prod\nolimits_{i=1}^n \langle \bb_i, \bx\rangle^{q_i}\Big \rangle_{\gaus_n}\;,\]
which is equivalent to showing that for any $\sum_{i=1}^m q_i\leq \ell$,
\[\E_{\bx\sim \p^{A'}_\bV} \left[\prod\nolimits_{i=1}^m\langle \bv_i,
\bx\rangle^{q_i} \right]
=\Big \langle \sum\nolimits_{k=0}^\ell
\langle\bA_k,\bH_k(\langle\bv_1,\bx\rangle,\ldots,\langle\bv_m,\bx\rangle)\rangle,\prod\nolimits_{i=1}^m \langle \bv_i, \bx\rangle^{q_i}\Big \rangle_{\gaus_m}\; ,\]
which is the same as for any $q=\sum_{i=1}^m q_i\leq \ell$, $\E_{\by\sim A'}[\bH_q(\by)]=\sum_{k=0}^\ell\E_{\by\sim\mathcal{N}_m}[\langle\bA_k,\bH_k(\by)\rangle\bH_q(\by)]$.
One can see that the above holds from the definition of 
$\bA_k$ and the orthornormal property of Hermite tensors.
Therefore, by Claim~\ref{clm:othor-tran}, we have that
\begin{align*}
\E_{\bx\sim \p^{A'}_\bV} [f^{\leq \ell}(\bx) ]
&=\left \langle \sum\nolimits_{k=0}^\ell\langle\bA_k,\bH_k(\langle\bv_1,\bx\rangle,\ldots,\langle\bv_m,\bx\rangle)\rangle, f^{\leq \ell}\right \rangle_{\gaus_n}\\
&=\left\langle\sum\nolimits_{k=0}^\ell\langle\bA_k,\bH_k(\bV^\intercal\bx)\rangle,f^{\le\ell}\right\rangle_{\mathcal{N}_n}\\
&=\left\langle\sum\nolimits_{k=0}^\ell\langle\bV^{\otimes k}\bA_k,\bH_k(\bx)\rangle,f^{\le\ell}\right\rangle_{\mathcal{N}_n}\\
&=\left \langle \sum\nolimits_{k=0}^\ell\langle \bV^{\otimes k}\bA_k, \bH_k(\bx)\rangle, \sum\nolimits_{k=0}^\ell \langle \bT_k,\bH_k(\bx)\rangle\right \rangle_{\gaus_n}\\
&=\sum\nolimits_{k=0}^\ell \langle \bV^{\otimes k}\bA_k, \bT_k \rangle\; ,
\end{align*}
where the last equality uses
the orthonormal property of Hermite tensors.
This completes the proof.
\end{proof}

\begin{remark}
{\em Ideally, in Lemma \ref{lem:hermite-decomposition}, we would like to have
$\E_{\bx\sim \p^{A'}_\bV} [f(\bx) ]
=\sum_{k=0}^\infty \langle \bV^{\otimes k}\bA_k, \bT_k \rangle$.
However, since we do not assume that $\chi^2(A',\gaus_m)<\infty$, 
this convergence may not hold.}
\end{remark}

Recall that our goal is to show that 
$|\E_{\bx\sim \p^A_\bV}[f(\bx)]-\E_{\bx\sim \gaus_n}[f(\bx)]|$ is small 
with high probability.
Observe that $\E_{\bx\sim \gaus_n}[f(\bx)]= \bT_0$ which is the first term in the summation of
$\sum_{k=0}^\ell \langle \bV^{\otimes k}\bA_k,\bT_k\rangle$ (since $\bA_0=1$).
Therefore, given Lemma \ref{lem:hermite-decomposition},
it suffices to show that 
$ |\sum_{k=1}^\ell \langle \bV^{\otimes k}\bA_k,\bT_k\rangle |$ and $ |\E_{\bx\sim \p^{A'}_\bV}[f^{\geq \ell}(\bx)] |$ 
are both small with high probability.
We ignore the $ |\E_{\bx\sim \p^{A'}_\bV}[f^{\geq \ell}(\bx)] |$ part for now, 
as this is mostly a technical issue.
To bound $|\sum_{k=1}^\ell \langle \bV^{\otimes k}\bA_k,\bT_k\rangle |$,
it suffices to analyze $\sum_{k=1}^\ell |\langle \bV^{\otimes k}\bA_k,\bT_k\rangle |$ by looking at each term $|\langle \bV^{\otimes k}\bA_k,\bT_k\rangle |= |\langle \bA_k,(\bV^{\intercal})^{\otimes k}\bT_k\rangle |\le\|\bA_k\|_2\|(\bV^{\intercal})^{\otimes k}\bT_k\|_2$. 
To show that the summation is small,  
we need to prove that (with high probability):

\begin{enumerate}
\item
$\|\bA_k\|_2$ does not grow too fast w.r.t $k$; 
\item
$\|(\bV^{\intercal})^{\otimes k}\bT_k\|_2$ decays very fast w.r.t $k$ (is small with high probability w.r.t the randomness of $\bV$).
\end{enumerate}

We proceed to establish these in turn below. 

\paragraph{$\|\bA_k\|_2$ does not grow too fast:}
We will use slightly different arguments 
depending on the size of $k$.
We consider three cases
: 
$k<d$, $d\leq k \leq n^{(1-\lambda)/4}$, and $k\geq n^{(1-\lambda)/4}$
(the value in the exponent will deviate by a small quantity to make the proof go through).
For $k<d$, $\|\bA_k\|_2$ grows slowly by the approximate moment-matching property of $A'$.
For $d\leq k\leq n^{(1-\lambda)/4}$, we require the following fact:
\begin{fact} \label{fct:hermite-upper-bound-medium-k}
Let $\bH_k$ be the $k$-th Hermite tensor for $m$ dimensions.
Suppose $\|\bx\|_2\le B$,
then $\|\bH_k(\bx)\|_2\le 2^{k}m^{k/4}B^{k}k^{-k/2}\exp\left(\binom{k}{2}/B^2\right)$.
\end{fact}

\noindent We provide the proof of Fact~\ref{fct:hermite-upper-bound-medium-k} in Appendix~\ref{sec:omit-Fourier}.

For $k>n^{(1-\lambda)/4}$,
we can show that $\|\bA_k\|_2$ does not grow too fast by the following asymptotic bound on 
Hermite tensors.

\begin{fact} \label{fct:hermite-upper-bound-large-k}
Let $\bH_k$ be the $k$-th Hermite tensor for $m$ dimensions. 
Then $$\|\bH_k(\bx)\|_2\leq 2^{O(m)} \binom{k+m-1}{m-1}^{1/2}\exp(\|\bx\|_2^2/4) \;.$$
\end{fact}

\noindent We provide the proof of Fact~\ref{fct:hermite-upper-bound-large-k} in Appendix~\ref{sec:omit-Fourier}.

\paragraph{$\|(\bV^{\intercal})^{\otimes k}\bT_k\|_2$ decays very fast:}
We show that $|\langle \bV^{\otimes k}, \bT_k \rangle |$ is small with high probability by bounding its $a$-th moment for some even $a$.
Notice that since $\|\bH_k\|_2\leq \|f\|_2\leq 1$, we 
can then combine it with the following lemma:

\begin{lemma} \label{lem:low-rank-analysis}
Let $k\in \Z_+$, $a\in \Z_+$ be even, $\bT\in\mathbb{R}^{n^k}$
and $m\in \Z_+$ satisfy $m<n$. 
Then there exists a unit vector $\bu\in\mathbb{S}^{n-1}$ such that
\[
\E_{\bV\sim U(\orthor_{n,m})}[\|(\bV^\intercal)^{\otimes k}\bT\|_2^a]\le\E_{\bV\sim U(\orthor_{n,m})}\left [\|\bV^\intercal\bu\|_2^{ak/2}\right ]
\|\bT\|_2^a \;.
\]
\end{lemma}
\begin{proof}
    Notice that 
    \begin{align*}
        \E_{\bV\sim U(\orthor_{n,m})}[\|(\bV^\intercal)^{\otimes k}\bT\|_2^a]
        =&\E_{\bV\sim U(\orthor_{n,m})}[\|(\bV^\intercal)^{\otimes ak/2}\bT^{a/2}\|_2^2]\\
        =&\E_{\bV\sim U(\orthor_{n,m})}[\langle \bV^{\otimes ak/2}(\bV^\intercal)^{\otimes ak/2},\bT^{\otimes a}\rangle]\\
        =&\langle\E_{\bV\sim U(\orthor_{n,m})}[\bV^{\otimes ak/2}(\bV^\intercal)^{\otimes ak/2}],\bT^{\otimes a}\rangle\\
        \leq & \|\E_{\bV\sim U(\orthor_{n,m})}[\bV^{\otimes ak/2}(\bV^\intercal)^{\otimes ak/2}]\|_2\|\bT\|_2^a \;.
    \end{align*}
    Therefore, it suffices to bound the spectral norm $\|\E_{\bV\sim U(\orthor_{n,m})}[\bV^{\otimes ak/2}(\bV^\intercal)^{\otimes ak/2}]\|_2$.

    Let $\bA=\E_{\bV\sim U(\orthor_{n,m})}[\bV^{\otimes ak/2}(\bV^\intercal)^{\otimes ak/2}]$, $\bT_0$ be the eigenvector associated with the largest absolute eigenvalue, and let $\bu=\mathrm{argmax}_{\bu\in \mathbb{S}^{n-1}} |\langle \bT_0, \bu^{\otimes ak/2}\rangle|$.
    Then, we have 
    \begin{align*}
        \|\bA\|_2 
        =&|\langle\bA\bT_0,\bu^{\otimes ak/2}\rangle|/|\langle \bT_0,\bu^{\otimes ak/2}\rangle|
        =|\langle\bT_0,\bA\bu^{\otimes ak/2}\rangle|/|\langle \bT_0,\bu^{\otimes ak/2}\rangle|\\
        =&|\langle\bT_0,\E_{\bV\sim U(\orthor_{n,m})}[(\bV\bV^{\intercal}\bu)^{\otimes ak/2}]\rangle|/|\langle \bT_0,\bu^{\otimes ak/2}\rangle|\\
        =&|\E_{\bV\sim U(\orthor_{n,m})}[\langle\bT_0,(\bV\bV^{\intercal}\bu)^{\otimes ak/2}\rangle]|/|\langle \bT_0,\bu^{\otimes ak/2}\rangle|\\
        \leq &\E_{\bV\sim U(\orthor_{n,m})}[|\langle\bT_0,(\bV\bV^{\intercal}\bu)^{\otimes ak/2}\rangle|]/|\langle \bT_0,\bu^{\otimes ak/2}\rangle|\\
        \leq & \E_{\bV\sim U(\orthor_{n,m})}[\|(\bV\bV^{\intercal}\bu)^{\otimes ak/2}\|_2|\langle \bT_0,\bu^{\otimes ak/2}\rangle|]/|\langle \bT_0,\bu^{\otimes ak/2}\rangle|\\
        = & \E_{\bV\sim U(\orthor_{n,m})}[\|(\bV\bV^{\intercal}\bu)^{\otimes ak/2}\|_2]\\
        =& \E_{\bV\sim U(\orthor_{n,m})}\left [\|\bV^\intercal\bu\|_2^{ak/2}\right ]\;,
    \end{align*}      
    where we use $\bu=\mathrm{argmax}_{\bu\in \mathbb{S}^{n-1}} |\langle \bT_0, \bu^{\otimes ak/2}\rangle|$ in the second inequality.
    Using Lemma \ref{lem:random-subspace-correlation-moment}, and plugging everything back, we get
$$\E_{\bV\sim U(\orthor_{n,m})}[\|(\bV^\intercal)^{\otimes k}\bT\|_2^a]\le\E_{\bV\sim U(\orthor_{n,m})}\left [\|\bV^\intercal\bu\|_2^{ak/2}\right ]
\|\bT\|_2^a\;.
$$
\end{proof}

Roughly speaking, this is just the $ak/2$-th moment of the correlation between 
a random subspace and a random direction, 
which can be bounded above by the following lemma and corollary 
(see Appendix~\ref{sec:omit-Fourier} for the proofs).

\begin{lemma} \label{lem:random-subspace-correlation-moment}
For any even $k\in \N$, and $\bu\in \mathbb{S}^{n-1}$,
$\E_{\bV^{\intercal}\sim U(\orthor_{n,m})}[\|\bV\bu\|_2^k]=\Theta\left (\frac{\Gamma\left (\frac{k+m}{2}\right )\Gamma\left (\frac{n}{2}\right )}
{\Gamma\left (\frac{k+n}{2}\right )\Gamma\left (\frac{m}{2}\right )}\right )$. 
\end{lemma}

\begin{corollary} \label{col:random-subspace-correlation-moment}
For any even $k\in \N$, and $\bu\in \mathbb{S}^{n-1}$,
\[
\E_{\bV\sim U(\orthor_{n,m})}[\|\bV^\intercal\bu\|_2^k]=O(2^{k/2}(n/\max(m,k))^{-k/2})\; .
\]
In addition, if there exists some constant $c\in (0,1)$ such that $m\leq n^c<k$, then
\[
\E_{\bV\sim U(\orthor_{n,m})}[\|\bV^\intercal\bu\|_2^k]=\exp(-\Omega(n^c\log n))O\left (\left (\frac{n^c+n}{k+n}\right )^{(n-m)/2}\right )\; .
\]
\end{corollary}

To combine the above results and give the high probability upper bound on
$\sum_{k=1}^{\ell} |\langle \bA_k,(\bV^{\intercal})^{\otimes k}\bT_k\rangle|$,
we require the following lemma.
\begin{lemma} \label{lem:summation-upper-bound}
    Under the conditions of Proposition \ref{prp:main-tail-bound},
    and further assuming
    $d,m\leq n^{\lambda}/\log n$,
    $\nu< 2$
    and $\left (\frac{\Gamma(d/2+m/2)}{\Gamma(m/2)}\right )n^{-((1-\lambda)/4-c) d}<2$, the following holds:
    For any $n$ that is at least a sufficiently large  
    constant depending on $c$,
    there is a 
    $B<n$ such that the truncated distribution $A'$,
    defined as the distribution of $\bx\sim A$ conditioned on $\bx\in \mathbb{B}^{m}(B)$, satisfies 
    $d_{\mathrm{TV}}(A,A')\leq \left (\frac{\Gamma(d/2+m/2)}{\Gamma(m/2)}\right )n^{-((1-\lambda)/4-c) d}$.
    Furthermore for any $\ell\in \N$,
    except with probability at most $2^{-n^{\Omega(c)}}$ w.r.t. $\bV\sim U(\orthor_{n,m})$, it holds 
    $    
    \sum_{k=1}^{\ell} |\langle \bA_k,(\bV^{\intercal})^{\otimes k}\bT_k\rangle|=\left (\frac{\Gamma(d/2+m/2)}{\Gamma(m/2)}\right )n^{-((1-\lambda)/4-c) d}+(1+o(1))\nu \, 
    $
   where $\bA_k=\E_{\bx\sim A'} [\bH_k(\bx)]$ and $\bT_k=\E_{\bx\sim \gaus_n} [f(\bx) \bH_k(\bx)]$. 
\end{lemma}
\begin{proof}
As we have discussed,
we will consider three ranges of $k$.
However, for some technical reasons and the ease of calculations, 
we will additionally break the second range into two ranges.
We can write
\begin{align*}
\sum_{k=1}^{\ell} |\langle \bA_k,(\bV^{\intercal})^{\otimes k}\bT_k\rangle|
=
&\sum_{k=1}^{d-1}| \langle \bA_k,(\bV^\intercal)^{\otimes k}\bT_k\rangle|
+\sum_{k=d}^{n^{\lambda}}| \langle \bA_k,(\bV^\intercal)^{\otimes k}\bT_k\rangle|\\
&+\sum_{k=n^{\lambda}+1}^{T}| \langle \bA_k,(\bV^\intercal)^{\otimes k}\bT_k\rangle|
+\sum_{k=T+1}^{\ell}| \langle \bA_k,(\bV^\intercal)^{\otimes k}\bT_k\rangle|\; ,
\end{align*}
where $T$ is a value we will later specify.
To analyze each 
$| \langle \bA_k,(\bV^\intercal)^{\otimes k}\bT_k\rangle|$,
recall that 
$| \langle \bA_k,(\bV^\intercal)^{\otimes k}\bT_k\rangle|
\leq 
\| \bA_k\|_2 \|(\bV^\intercal)^{\otimes k}\bT_k\|_2$, 
where $\bA_k=\E_{\bx\sim A'} [\bH_k(\bx)]$ is a constant (not depending
on the randomness of $\bV$).
For $\|(\bV^\intercal)^{\otimes k}\bT_k\|_2$, 
we can show it is small by bounding its $a$-th moment for even $a$ using Lemma \ref{lem:low-rank-analysis} which says
\begin{align*}
\E_{\bV\sim U(\orthor_{n,m})}[\|(\bV^\intercal)^{\otimes k}\bT\|_2^a]\le\E_{\bV\sim U(\orthor_{n,m})}\left [\|\bV^\intercal\bu\|_2^{ak/2}\right ]
\|\bT\|_2^a\; ,
\end{align*}
for some unit vector $\bu\in\mathbb{S}^{n-1}$. 
We will apply this strategy on the four different ranges of $k$.

Without loss of generality, we will assume that 
$\lambda\geq 4c$.
Suppose $\lambda\geq 4c$ is not true,
then we can simply consider a new pair $\lambda', c'$,
where $\lambda'=\lambda+2c$ and $c'=c/2$.
Notice that $(1-\lambda)/4-c=(1-\lambda')/4-c'$; 
therefore, the SQ lower bound in the statement remains unchanged.

We start by picking the following parameters (the sufficiently close here only depends on $c$):
\begin{itemize}[leftmargin=*]
    \item We require $d,m\leq n^{\lambda}/\log n$;
    \item $B=n^{\alpha}$ where $\alpha<(1-\lambda_3)/4$ is sufficiently close;
    \item $T=c_1 B^2$ where $c_1$ is a sufficiently large constant;
    \item We let $\lambda_3>\lambda_2>\lambda_1>\lambda$ to be sufficiently close (the difference between these quantities will be a sufficiently small constant fraction of $c$);
\end{itemize}
We now bound the summation $\sum_{k=1}^{\ell} |\langle \bA_k,(\bV^{\intercal})^{\otimes k}\bT_k\rangle|$ as follows:

\begin{itemize}[leftmargin=*]
\item {\bf $\sum_{k=1}^{d-1} |\langle \bA_k,(\bV^{\intercal})^{\otimes k}\bT_k\rangle|$
is small with high probability:} 

Since $\left(\frac{\Gamma(d+m/2)}{\Gamma(m/2)}\right)n^{-((1-\lambda)/4-c) d}<2$ and $B=n^{\alpha}$, where $\alpha$ is sufficiently close to
$(1-\lambda)/4$,
the parameters satisfy the condition
$B^d=\omega\left (2^{d/2}\sqrt{\frac{\Gamma(d+m/2)}{\Gamma(m/2)}}\right )$
in Lemma \ref{lem:truncation}.
Since $k< d$, 
by Lemma \ref{lem:truncation}, 
we have
\begin{align*}
\|\bA_k\|_2&=\|\E_{\bx\sim A'}[\bH_k(\bx)]-\E_{\bx\sim \gaus_m}[\bH_k(\bx)]\|_2\\&=2^{O(k)}\left (2^{d/2}\sqrt{\frac{\Gamma(d+m/2)}{\Gamma(m/2)}}\right )B^{-(d-k)} 
  +\left (1+O\left (2^{d/2}\sqrt{\frac{\Gamma(d+m/2)}{\Gamma(m/2)}}\right )B^{-d}\right )\nu\; .
\end{align*}
Let $a$ be the largest even number such that $ak/2\leq n^{\lambda}$, where $d=o(n^{\lambda})$ implies $a\geq 2$.
Then using Lemma~\ref{lem:low-rank-analysis} and Corollary~\ref{col:random-subspace-correlation-moment}, we have 
\begin{align*}  
\E_{\bV\sim U(\orthor_{n,m})}[\|(\bV^{\intercal})^{\otimes k}\bT_k\|_2^a]
=& \E_{\bV\sim U(\orthor_{n,m})}\left [\|\bV^\intercal\bu\|_2^{ak/2}\right ]
\|\bT_k\|_2^a
\le \E_{\bV\sim U(\orthor_{n,m})}\left [\|\bV^\intercal\bu\|_2^{ak/2}\right ]\\
=& O(2^{ak/4}n^{-(1-\lambda)ak/4})
= O(n^{-(1-\lambda_1)ak/4})\; .
\end{align*}
Using Markov's Inequality, this implies the tail bound
\[
\pr[\|(\bV^{\intercal})^{\otimes k}\bT_k\|_2\geq n^{-(1-\lambda_2)k/4}]\leq 2^{-\Omega(cn^{\lambda})}=2^{-n^{\Omega(c)}}\; .
\]
Therefore, we have
\begin{align*}
    &\sum_{k=1}^{d-1}| \langle \bA_k,(\bV^{\intercal})^{\otimes k}\bT_k\rangle|
    \leq \sum_{k=1}^{d-1} \|\bA_k\|_2\|(\bV^{\intercal})^{\otimes k}\bT_k\|_2\\
    \leq &\sum_{k=1}^{d-1}  n^{-(1-\lambda_2)k/4}\left (2^{O(k)}\left (2^{d/2}\sqrt{\frac{\Gamma(d+m/2)}{\Gamma(m/2)}}\right )B^{-(d-k)}+\left (1+O\left (2^{d/2}\sqrt{\frac{\Gamma(d+m/2)}{\Gamma(m/2)}}\right )B^{-d}\right )\nu\right )\\
    \leq  & (1+o(1))\left (2^{d/2}\sqrt{\frac{\Gamma(d+m/2)}{\Gamma(m/2)}}B^{-d}+\nu\right )
    =  (1+o(1))\left (2^{d/2}\sqrt{\frac{\Gamma(d+m/2)}{\Gamma(m/2)}}n^{-\alpha d}+\nu\right )\; ,
\end{align*}
except with probability $2^{-n^{\Omega(c)}}$,
where the second inequality follows from $B=n^{\alpha}=o(n^{(1-\lambda)/4})$.

\item {\bf $\sum_{k=d}^{n^{\lambda}} |\langle \bA_k,(\bV^{\intercal})^{\otimes k}\bT_k\rangle|$ is small with high probability:}

In the previous case, we have argued that the parameters satisfy the condition
$B^d= \omega\left (2^{d/2}\sqrt{\frac{\Gamma(d+m/2)}{\Gamma(m/2)}}\right )$
in Lemma~\ref{lem:truncation}.
Since $k\geq d$, by Lemma~\ref{lem:truncation}, we have\[
\|\bA_k\|_2
=\|\E_{\bx\sim A'}[\bH_k(\bx)]-\E_{\bx\sim \gaus_m}[\bH_k(\bx)]\|_2
=2^{O(k)}\left (2^{d/2}\sqrt{\frac{\Gamma(d+m/2)}{\Gamma(m/2)}}\right ) B^{k-d}\; .
\]
Let $a$ be the largest even number that $ak/2\leq n^{\lambda}$, where $d= o(n^{\lambda})$ implies $a\geq 2$.
Applying Lemma~\ref{lem:low-rank-analysis} and Corollary~\ref{col:random-subspace-correlation-moment} yields
\begin{align*}  
\E_{\bV\sim U(\orthor_{n,m})}[\|(\bV^{\intercal})^{\otimes k}\bT_k\|_2^a]
=& \E_{\bV\sim U(\orthor_{n,m})}\left [\|\bV^\intercal\bu\|_2^{ak/2}\right ]
\|\bT_k\|_2^a
\le \E_{\bV\sim U(\orthor_{n,m})}\left [\|\bV^\intercal\bu\|_2^{ak/2}\right ]\\
=& O(2^{ak/4}n^{-(1-\lambda)ak/4})
= O(n^{-(1-\lambda_1)ak/4})\; .
\end{align*}
Therefore, we have
\begin{align*}
    \sum_{k=d}^{n^{\lambda}} |\langle \bA_k,(\bV^{\intercal})^{\otimes k}\bT_k\rangle|
    \leq &\sum_{k=d}^{n^{\lambda}} \|\bA_k\|_2\|(\bV^{\intercal})^{\otimes k}\bT_k\|_2
    \leq \sum_{k=d}^{n^{\lambda}}  n^{-(1-\lambda_2)k/4}2^{O(k)}\left (2^{d/2}\sqrt{\frac{\Gamma(d+m/2)}{\Gamma(m/2)}}\right ) B^{k-d}\\
    = & 2^{O(d)}n^{-((1-\lambda_2)/4)d}
    =  n^{-((1-\lambda_3)/4)d}\; ,
\end{align*}
except with probability $2^{-n^{\Omega(c)}}$ (the first equality above follows from $B=n^{\alpha}=o(n^{(1-\lambda)/4})=o(n^{(1-\lambda_2)/4})$).

\item {\bf $\sum_{k=n^\lambda+1}^{T} |\langle \bA_k,(\bV^{\intercal})^{\otimes k}\bT_k\rangle|$ is small with high probability:}

We can assume without loss of generality that $n^\lambda<T$, because otherwise this term is $0$.
Using Lemma~\ref{fct:hermite-upper-bound-medium-k}, we have
\begin{align*}
\|\bA_k\|_2
=&\|\E_{\bx\sim A'}[\bH_k(\bx)]-\E_{\bx\sim \gaus_m}[\bH_k(\bx)]\|_2
=\|\E_{\bx\sim A'}[\bH_k(\bx)]\|_2\\
\leq &2^{k}m^{k/4}B^{k}k^{-k/2}\exp\left(\binom{k}{2}/B^2\right)
\leq 2^{O(k)}m^{k/4}B^{k}k^{-k/2} \; .
\end{align*}
Then let $a$ be the largest even number that $ak/2\leq T$, where $k\leq T$ implies $a\geq 2$.
Applying Lemma~\ref{lem:low-rank-analysis} and Corollary~\ref{col:random-subspace-correlation-moment} yields
\begin{align*}  
\E_{\bV\sim U(\orthor_{n,m})}[\|(\bV^{\intercal})^{\otimes k}\bT_k\|_2^a]
=& \E_{\bV\sim U(\orthor_{n,m})}\left [\|\bV^\intercal\bu\|_2^{ak/2}\right ]
\|\bT_k\|_2^a
\le \E_{\bV\sim U(\orthor_{n,m})}\left [\|\bV^\intercal\bu\|_2^{ak/2}\right ]\\
=& O(2^{ak/4}(ak/2n)^{ak/4})
= O\left(\frac{n}{ak}\right)^{-ak/4}\; ,
\end{align*}
which implies the tail bound
\[
\pr[\|(\bV^{\intercal})^{\otimes k}\bT_k\|_2\geq n^{-((1-\lambda)/4) k}k^{k/4}]\le n^{-\lambda ak/4}\leq 2^{-\Omega(T)}=2^{-\Omega(n^{2\alpha})}=2^{-n^{\Omega(c)}}\; .
\]
Therefore, we have 
\begin{align*}
    \sum_{k=n^{\lambda}+1}^{T} |\langle \bA_k,(\bV^{\intercal})^{\otimes k}\bT_k\rangle|
    \leq &\sum_{k=n^{\lambda}+1}^{T} \|\bA_k\|_2\|(\bV^{\intercal})^{\otimes k}\bT_k\|_2
    \leq \sum_{k=n^{\lambda}+1}^{T}  2^{O(k)}m^{k/4}B^{k}k^{-k/4} n^{-((1-\lambda)/4) k}\\
    = & O\left (B^{n^{\lambda}+1}n^{-((1-\lambda)/4) (n^{\lambda}+1)}\right)
    =  n^{-\Omega(cn^{\lambda})}
    =  n^{-\Omega(cd\log n)}
    \le  n^{-d}\; ,
\end{align*}
except with probability $2^{-n^{\Omega(c)}}$.

\item {\bf $\sum_{k=T+1}^{\ell} |\langle \bA_k,(\bV^{\intercal})^{\otimes k}\bT_k\rangle|$
is small with high probability:}

Combining Fact~\ref{fct:hermite-upper-bound-large-k} with the fact 
that $A'$ is bounded inside $\mathbb{B}^{m}(B)$,
we have that
\begin{align*}
\|\bA_k\|_2
=& \|\E_{\bx\sim A'}[\bH_k(\bx)]-\E_{\bx\sim \gaus_m}[\bH_k(\bx)]\|_2
=\|\E_{\bx\sim A'}[\bH_k(\bx)]\|_2
\leq  2^{O(m)}\binom{k+m-1}{m-1}^{1/2}\exp(B^2/4)\; .
\end{align*}
We pick $a=2$.
Note that $ak/2> T=c_1n^{2\alpha}\ge n^{2\alpha}$. Applying Lemma~\ref{lem:low-rank-analysis} and Corollary~\ref{col:random-subspace-correlation-moment} yields
\begin{align*}  
\E_{\bV\sim U(\orthor_{n,m})}[\|(\bV^{\intercal})^{\otimes k}\bT_k\|_2^a]
=& \E_{\bV\sim U(\orthor_{n,m})}\left [\|\bV^\intercal\bu\|_2^{ak/2}\right ]
\|\bT_k\|_2^a
= \E_{\bV\sim U(\orthor_{n,m})}\left [\|\bV^\intercal\bu\|_2^{ak/2}\right ]\\
=& \exp(-\Omega(n^{2\alpha}\log n))O\left (\left (\frac{n^{2\alpha}+n}{k+n}\right )^{(n-m)/2}\right )\; .
\end{align*}
Applying Markov's inequality yields the tail bound
\[
\pr\left [\|(\bV^{\intercal})^{\otimes k}\bT_k\|_2\geq 2^{-\Omega(n^{2\alpha}\log n)}O\left (\left (\frac{n^{2\alpha}+n}{k+n}\right )^{(n-m)/4}\right )\right ]\leq 2^{-\Omega(n^{2\alpha})}=2^{-\Omega(n^{2c})}=2^{-n^{\Omega(c)}}\; .
\]
Therefore, we have
\begin{align*}
    \sum_{k=T+1}^{\ell} |\langle \bA_k,(\bV^{\intercal})^{\otimes k}\bT_k\rangle|
    \leq &\sum_{k=T+1}^{\infty} |\langle \bA_k,(\bV^{\intercal})^{\otimes k}\bT_k\rangle|
    \leq \sum_{k=T+1}^{\infty} \|\bA_k\|_2\|(\bV^{\intercal})^{\otimes k}\bT_k\|_2\\
    \leq &\sum_{k=T+1}^{\infty} 2^{O(m)} \binom{k+m-1}{m-1}^{1/2}\exp(B^2/4)2^{-\Omega(n^{2\alpha}\log n)}O\left (\left (\frac{n^{2\alpha}+n}{k+n}\right )^{(n-m)/4}\right )\\
    \leq &\sum_{k=T}^{\infty}  2^{-\Omega(n^{2\alpha}\log n)} \left (\frac{k+m}{T+m}\right )^{m/2}
    \left ( \frac{n^{2\alpha}+n}{k+n}\right )^{n/8}\;,
\end{align*}
where the last inequality follows from our choice of parameters. Therefore, we have that
\begin{align*}
\sum_{k=T+1}^{\ell} |\langle \bA_k,(\bV^{\intercal})^{\otimes k}\bT_k\rangle|
    \leq &\sum_{k=T}^{\infty}  2^{-\Omega(n^{2\alpha}\log n)} \left (1+\frac{k-T}{T+m}\right )^{m/2}
    \left ( 1+\frac{k-n^{2\alpha}}{n^{2\alpha}+n}\right )^{-n/8}\\
    \leq &\sum_{k=T}^{\infty}  2^{-\Omega(n^{2\alpha}\log n)} \left (1+\frac{k-T}{n^{2\alpha}+n}\right )^{(m/2)(2n/T)}
    \left ( 1+\frac{k-T}{n^{2\alpha}+n}\right )^{-n/8}\\
    \leq &\sum_{k=T}^{\infty}  2^{-\Omega(n^{2\alpha}\log n)} \left ( 1+\frac{k-T}{n^{2\alpha}+n}\right )^{-n/8+nm/T}\\
    \leq &\sum_{k=T}^{\infty}  2^{-\Omega(n^{2\alpha}\log n)} \left (\frac{n^{2\alpha}+n-T+k}{n^{2\alpha}+n}\right )^{-n/16}\\
    \leq &2^{-\Omega(n^{2\alpha}\log n)} \int_{k=T-1}^{\infty}   \left (\frac{n^{2\alpha}+n-T+k}{n^{2\alpha}+n}\right )^{-n/16}dk\\
    =&\frac{2^{-\Omega(n^{2\alpha}\log n)}(n^{2\alpha}+n)^{n/16}}{(n/16-1)(n^{2\alpha}+n-1)^{n/16-1}}
    \\
    =&2^{-\Omega(n^{2\alpha})}\; ,
\end{align*}
except with probability $2^{-n^{\Omega(c)}}$.

\end{itemize}

Adding the four cases above together, we get for any $d,m\leq n^{\lambda}/\log n$ and $n$ is at least a sufficiently large constant depending on $c$,
\begin{align*}
\sum_{k=1}^{\ell} |\langle \bA_k,(\bV^{\intercal})^{\otimes k}\bT_k\rangle|
\leq &
(1+o(1))\left (2^{d/2}\sqrt{\frac{\Gamma(d+m/2)}{\Gamma(m/2)}}n^{-\alpha d}+\nu\right )
 +n^{-((1-\lambda_3)/4)d}
 + n^{-d}
 +2^{-\Omega(n^{2\alpha})}\\
\leq & \left (\frac{\Gamma(d/2+m/2)}{\Gamma(m/2)}\right )n^{-((1-\lambda)/4-c/2) d}+(1+o(1))\nu\\
= & \left (\frac{\Gamma(d/2+m/2)}{\Gamma(m/2)}\right )n^{-((1-\lambda)/4-c) d}+(1+o(1))\nu\;,
\end{align*}
except with probability $2^{-n^{\Omega(c)}}$,
where the last inequality above follows from $\frac{\Gamma(d/2+m/2)}{\Gamma(m/2)}\geq 2^{-O(d)}\sqrt{\frac{\Gamma(d+m/2)}{\Gamma(m/2)}}$.
\end{proof}

\subsection{Proof for Proposition~\ref{prp:main-tail-bound}}\label{ssec:main-tail-bound}

We are now ready to prove Proposition~\ref{prp:main-tail-bound} which 
is the main technical ingredient of our main result.
Proposition~\ref{prp:main-tail-bound} states that 
$ |\E_{\bx\sim \p^A_\bV}[f(\bx)]-\E_{\bx\sim \gaus_n}[f(\bx)] |$
is small with high probability.
The main idea of the proof is to use Fourier analysis on $\E_{\bx\sim \p^{A'}_\bV}[f(\bx)]$ as we discussed in the last section, where $A'$ is the the distribution obtained by truncating $A$ inside $\mathbb{B}^{m}(B)$.

\begin{proof} [Proof for Proposition~\ref{prp:main-tail-bound}]
For convenience, we let $\zeta=(1-\lambda)/4-c$.
We will analyze $\E_{\bx \sim A}[\bH_k(\bx)]$ by truncating $A$.
Therefore, we will apply Lemma~\ref{lem:summation-upper-bound} here.
However, notice that Lemma~\ref{lem:summation-upper-bound} additionally assumes
$d,m\leq n^{\lambda}/\log n$,
$\nu< 2$ and $\left (\frac{\Gamma(d/2+m/2)}{\Gamma(m/2)}\right )n^{-\zeta d}<2$.
We show that all these three conditions can be assumed true without loss of generality. 
If either the second or the third condition is not true,
then our lower bound here is trivialized and is always true since $f$ is bounded between $[-1,+1]$.
For $d,m\leq n^{\lambda}/\log n$, consider a $\lambda'>\lambda$ such that $(1-\lambda')/4-\zeta=\frac{(1-\lambda)/4-\zeta}{2}$.
Then it is easy to see for any sufficiently large $n$ depending on $(1-\lambda)/4-\zeta$,
we have $d,m\leq n^{\lambda'}/\log n$ and $\zeta\leq (1-\lambda)/4-\zeta$.
Therefore, we can apply Lemma~\ref{lem:summation-upper-bound} for $\lambda'$.

Now let $B=n^{\alpha}$ where $\alpha<(1-\lambda)/4$ is the constant in Lemma~\ref{lem:summation-upper-bound}.
Then we consider the truncated distribution $A'$ defined as the distribution of $\bx\sim A$ conditioned on 
$\bx\in \mathbb{B}^{m}(B)$.
By Lemma~\ref{lem:summation-upper-bound}, we have
$d_{\mathrm{TV}}(A,A')\leq \left (\frac{\Gamma(d/2+m/2)}{\Gamma(m/2)}\right )n^{-\zeta d}$.
Given that $f$ is bounded between $[-1, 1]$,
this implies $ |\E_{\bx\sim \p^A_\bV}[f(\bx)]-\E_{\bx\sim \p^{A'}_\bV}[f(\bx)] |\le2\dtv(\p^A_\bV,\p^{A'}_{\bV'})=2\dtv(A,A')\leq 2\left (\frac{\Gamma(d/2+m/2)}{\Gamma(m/2)}\right )n^{-\zeta d}$.
Thus, it suffices for us to analyze $\E_{\bx\sim \p^{A'}_\bV}[f(\bx)]$
instead of $\E_{\bx\sim \p^A_\bV}[f(\bx)]$.

Let $\ell=\ell_f(n)\in \N$ be a function depending only 
on the query function $f$ and the dimension $n$ 
($\ell$ to be specified later).
By Lemma \ref{lem:hermite-decomposition}, we have that 
$$\E_{\bx\sim \p^{A'}_\bV}[f(\bx)]= \sum_{k=0}^\ell|\langle \bA_k,(\bV^{\intercal})^{\otimes k}\bT_k\rangle|
+\E_{\bx\sim \p^{A'}_\bV}[f^{>\ell}(\bx)]\;.$$
Recall that we want to bound 
$ |\E_{\bx\sim \p^{A'}_\bV}[f(\bx)]-\E_{\bx\sim \gaus_n}[f(\bx)] |$ 
with high probability,
where we note that  
$\E_{\bx\sim \gaus_n}[f(\bx)]= \langle \bA_0, \bT_0\rangle$.
Therefore, we can write 
\begin{align*}
|\E_{\bx\sim \p^{A'}_\bV}[f(\bx)]-\E_{\bx\sim \gaus_n}[f(\bx)]|\leq 
\left |\sum\nolimits_{k=1}^\ell \langle \bA_k,(\bV^{\intercal})^{\otimes k}\bT_k\rangle\right |+
 |\E_{\bx\sim \p^{A'}_\bV}[f^{> \ell}(\bx)] |\; .
\end{align*}
For the first term, by Lemma \ref{lem:summation-upper-bound}, we have that $$\left| \sum_{k=1}^\ell\langle \bA_k,(\bV^{\intercal})^{\otimes k}\bT_k\rangle \right|
=\left (\frac{\Gamma(d/2+m/2)}{\Gamma(m/2)}\right )n^{-\zeta d}+(1+o(1))\nu \;,$$
except with probability $2^{-n^{\Omega(c)}}$. 

It now remains for us to show that 
$ |\E_{\bx\sim \p^{A'}_\bV}[f^{> \ell}(\bx)] |$
is also small with high probability.
Consider the distribution $D=\E_{\bv\sim U(\orthor_{n,m})} [\p^{A'}_\bV ]$. 
The following lemma shows
that $D$ is continuous and $\chi^2(D,\gaus_n)$
is at most a constant only depending on $n$ 
(independent of the choice of the distribution $A$).

\begin{lemma} \label{lem:high-degree-finite-chi-square}
Let $A$ be any distribution supported on $\mathbb{B}^{m}(n)$ 
for $n\in \N$ which is at least a sufficiently large universal constant.
Let $D=\E_{\bV\sim U(\orthor_{n,m})}[\p^{A'}_\bV ]$.
Then, $D$ is a continuous distribution and 
$\chi^2(D,\mathcal{N}_n) = O_n(1)$.
\end{lemma}

\noindent Roughly speaking, the proof of the lemma follows by noting 
that the average over $\bV$ of $\p_\bV^A$ is spherically symmetric.
We defer its proof to Appendix~\ref{ssec:omit-main-tail-bound}. 
Using this lemma, we can write 
\begin{align*}
\E_{\bV\sim U(\orthor_{n,m})}\big [\big|\E_{\bx\sim \p^{A'}_\bV}[f^{>\ell}(\bx)]\big |\big ]
&\leq \E_{\bV\sim U(\orthor_{n,m})}\big[\E_{\bx\sim \p^{A'}_\bV}\big [|f^{> \ell}(\bx) |\big ]\big ]=\E_{\bx\sim D}\big [|f^{> \ell}(\bx) |\big ]
\\&= \E_{\bx\sim \gaus_n} \left[\frac{D(\bx)}{\gaus_n(\bx)}\;  |f^{> \ell}(\bx) | \right]
\leq \sqrt{1+\chi^2(D,\gaus_n)}\; \|f^{> \ell} \|_2 \;,
\end{align*}
where the first equality holds due to Fubini's theorem.
Furthermore, since $\chi^2(D,\gaus_n)=O_n(1)$,
there is a function $\delta:\R\to\R$ such that $1+\chi^2(D,\gaus_n)\leq \delta(n)$.
Therefore, we have that
\begin{align*}
\E_{\bV\sim U(\orthor_{n,m})}\big[\big|\E_{\bx\sim \p^{A'}_\bV} [f^{>\ell}(\bx)]\big|\big]&\leq \sqrt{1+\chi^2(D,\gaus_n)}\;\|f^{>\ell} \|_2\leq  \delta(n)\|f^{>\ell}\|_2\; .
\end{align*}
We can take $\ell=\ell_f(n)$ ($\ell$ only depends on the query function $f$ and dimension $n$) to be a sufficiently large function 
such that 
$\|f^{> \ell} \|_2\leq 
\left(\frac{e^{-n}}{\delta(n)}\right)\left (\frac{\Gamma(d/2+m/2)}{\Gamma(m/2)}\right )n^{-\zeta d}$. 
Then we get
\[\E_{\bV\sim U(\orthor_{n,m})}\big [\big |\E_{\bx\sim \p^{A'}_\bV} [f^{> \ell}(\bx) ]\big |\big ]
\leq  \delta(n) \|f^{> \ell} \|_2
\leq e^{-n}\left (\frac{\Gamma(d/2+m/2)}{\Gamma(m/2)}\right )n^{-\zeta d}\; .\]
This gives the tail bound $\pr_{\bV\sim U(\orthor_{n,m})}\big [\big |\E_{\bx\sim \p^{A'}_\bV} [f^{> \ell}(\bx) ]\big |\geq \left (\frac{\Gamma(d/2+m/2)}{\Gamma(m/2)}\right )n^{-\zeta d}\big ]\leq e^{-n}$.

Using the above upper bounds, we have 
\begin{align*}
\big |\E_{\bx\sim \p^{A'}_\bV}[f(\bx)]-\E_{\bx\sim \gaus_n}[f(\bx)]\big |
&\leq 
\Big|\littlesum\nolimits_{k=1}^\ell  \langle \bA_k,(\bV^{\intercal})^{\otimes k}\bT_k\rangle\Big |+
 |\E_{\bx\sim \p^{A'}_\bV}[f^{> \ell}(\bx)] |
\\&=2\left (\frac{\Gamma(d/2+m/2)}{\Gamma(m/2)}\right )n^{-\zeta d}+(1+o(1))\nu\; ,
\end{align*}
except with probability $2^{-n^{\Omega(1)}}$. 
As we have argued at the beginning of the proof,
\begin{align*}
|\E_{\bx\sim \p^A_\bV}[f(\bx)]-\E_{\bx\sim \p^{A'}_\bV}[f(\bx)]|
&\le 2\left (\frac{\Gamma(d/2+m/2)}{\Gamma(m/2)}\right )n^{-\zeta d} \; .
\end{align*}
Therefore,
\begin{align*}
|\E_{\bx\sim \p^A_\bV}[f(\bx)]-\E_{\bx\sim \gaus_n}[f(\bx)]|
&\le 3\left (\frac{\Gamma(d/2+m/2)}{\Gamma(m/2)}\right )n^{-\zeta d}+(1+o(1))\nu\; ,
\end{align*}
except with probability $2^{-n^{\Omega(1)}}<2^{-n^{\Omega(c)}}$ given $c=O(1)$. 

In the end, notice that the above argument is still true if we take $\zeta'>\zeta$ such that $(1-\lambda)/4-\zeta'=\frac{(1-\lambda)/4-\zeta}{2}$.
Using the above argument for $\zeta'$ and given $n$ is a sufficiently large constant depending on $(1-\lambda)/4-\zeta=2((1-\lambda)/4-\zeta')$, we get
\begin{align*}
|\E_{\bx\sim \p^A_\bV}[f(\bx)]-\E_{\bx\sim \gaus_n}[f(\bx)]|
&\le \left (\frac{\Gamma(d/2+m/2)}{\Gamma(m/2)}\right )n^{-\zeta d}+(1+o(1))\nu\; ,
\end{align*}
except with probability $2^{-n^{\Omega((1-\lambda)/4-\zeta')}}=2^{-n^{\Omega(c)}}$. 
Replaceing $\zeta$ with $(1-\lambda)/4-c$ completes the proof.
\end{proof}

\bibliographystyle{alpha}

\bibliography{allrefs}

\appendix

\newpage

\section*{Appendix}

\section{Omitted Proofs from Section~\ref{sec:prelims}}\label{sec:omit-prelim}

\subsection{Proof of Claim~\ref{clm:othor-tran}}
\begin{proof}
The proof of Claim~\ref{clm:othor-tran} is obtained via the 
following calculation, 
using the definition of Hermite tensor (Definition \ref{def:Hermite-tensor}).
We will use $i,j$ for indexes in $[d]$.
\begin{align*}
\sqrt{k!}\bH_k(\bB\bx)_{i_1,\ldots,i_k}=&\sum_{\substack{\text{Partitions $P$ of $[k]$}\\ \text{into sets of size 1 and 2}}} \bigotimes_{\{a,b\}\in P} (-\bI_{i_a,i_b})\bigotimes_{\{c\}\in P} (\bB\bx)_{i_c}\\
=&\sum_{\substack{\text{Partitions $P$ of $[k]$}\\ \text{into sets of size 1 and 2}}} \bigotimes_{\{a,b\}\in P} (-(\bB\bI\bB^T)_{i_a,i_b})\bigotimes_{\{c\}\in P} (\bB\bx)_{i_c}\\
=&\sum_{\substack{\text{Partitions $P$ of $[k]$}\\ \text{into sets of size 1 and 2}}} 
\bigotimes_{\{a,b\}\in P} \left (-\left (\sum_{j_a,j_b=1}^n \bB_{i_a,j_a}\bI_{j_a,j_b}\bB^\intercal_{j_b,i_b}\right)\right)\bigotimes_{\{c\}\in P} \left (\sum_{j_c=1}^n \bB_{i_c,j_c}\bx_{j_c}\right )\\
=&\sum_{\substack{\text{Partitions $P$ of $[k]$}\\ \text{into sets of size 1 and 2}}} 
\sum_{j_1,\ldots,j_k=1}^n\bigotimes_{\{a,b\}\in P} \left (\bB_{i_a,j_a}(-\bI_{j_a,j_b})\bB^\intercal_{j_b,i_b}\right)\bigotimes_{\{c\}\in P} \left (\bB_{i_c,j_c}\bx_{j_c}\right )\\
=&\sum_{\substack{\text{Partitions $P$ of $[k]$}\\ \text{into sets of size 1 and 2}}} 
\sum_{j_1,\ldots,j_k=1}^n\bigotimes_{l\in [k]} \bB_{i_l,j_l}\bigotimes_{\{a,b\}\in P} (-\bI_{j_a,j_b})\bigotimes_{\{c\}\in P} \bx_{j_c} \\
=&\sum_{j_1,\ldots,j_k=1}^n\bigotimes_{l\in [k]} \bB_{i_l,j_l}\sum_{\substack{\text{Partitions $P$ of $[k]$}\\ \text{into sets of size 1 and 2}}} 
\bigotimes_{\{a,b\}\in P} (-\bI_{j_a,j_b})\bigotimes_{\{c\}\in P} \bx_{j_c}\\
=&\sum_{j_1,\ldots,j_k=1}^n\bigotimes_{l\in [k]} \bB_{i_l,j_l}\sqrt{k!}\bH_k(\bx)_{j_1,\ldots,j_k}\\
=&\sqrt{k!}\sum_{j_1,\ldots,j_k=1}^n (\bB^{\otimes k})_{i_1,\ldots,i_k,j_1,\ldots,j_k} \bH_k(\bx)_{j_1,\cdots,j_k}\; ,
\end{align*}
where the fourth and fifth equalities follow from the fact that $P$ is a partition of $[k]$, 
so changing the order of summation and multiplication gives exactly 
$\sum_{j_1,\ldots,j_k=1}^n$ and $\bigotimes_{l\in [k]}\bB_{i_l,j_l}$.
The seventh equality follows from the definition of the Hermite tensor.
The above is equivalent to
$\bH_k(\bB\bx)=\bB^{\otimes k}\bH_k(\bx)$. This completes the proof.
\end{proof}

\section{Omitted Proofs from Section~\ref{sec:ngca-sq}}

\subsection{Omitted Proofs from Section~\ref{ssec:Fourier}}\label{sec:omit-Fourier}

\subsubsection{Proof of Lemma~\ref{lem:truncation}}\label{truncation}

\begin{proof}
We construct the truncated distribution $A'$ as follows.
We first sample $\bx\sim A$, then we reject $\bx$ unless $\|\bx\|_2\leq B$. 
Let $A'$ be the distribution of the samples obtained from this process.

First notice that we will use the moment bound to bound the total variation distance, as follows 
\begin{align*}
\E_{\bx\sim A}[\|\bx\|_2^{d}]
\le &\E_{\bx\sim \gaus_m}[\|\bx\|_2^d]
+\nu\E_{\bx\sim \gaus_m}[\|\bx\|_2^{2d}]^{1/2}\\
= &\E_{t\sim \chi^2(m)}[t^{d/2}]+\nu\E_{t\sim \chi^2(m)}[t^{d}]^{1/2}\\
= &2^{d/2}\frac{\Gamma((d+m)/2)}{\Gamma(m/2)}+2^{d/2}\sqrt{\frac{\Gamma((2d+m)/2)}{\Gamma(m/2)}}\nu\\
\leq & c_2\left (2^{d/2}\sqrt{\frac{\Gamma(d+m/2)}{\Gamma(m/2)}}\right )\; ,
\end{align*}
where $c_2$ is a universal constant. 
Using Markov's inequality and the union bound, we have
\[ \pr_{\bx\sim A}[\bx\not\in \mathbb{B}^{m}(B)]\leq c_2\left (2^{d/2}\sqrt{\frac{\Gamma(d+m/2)}{\Gamma(m/2)}}\right )B^{-d} \;. \] 
By the definition of $A'$, we have that $d_{\mathrm{TV}}(A,A')\le c_2\left (2^{d/2}\sqrt{\frac{\Gamma(d+m/2)}{\Gamma(m/2)}}\right )B^{-d}$.

Then it only remains to verify 
$\|\E_{\bx \sim A'}[\bH_k(\bx)]-\E_{\bx\sim \gaus_m}[\bH_k(\bx)] \|_2$ for any $k<d$ and $k\geq d$ respectively.
For $k=0$, it is immediate that $\|\E_{\bx \sim A'}[\bH_k(\bx)]-\E_{\bx\sim \gaus_m}[\bH_k(\bx)] \|_2=0$.
Therefore, we only consider $1\leq k<d$.
We first look at 
$\|\E_{\bx \sim A}[\bH_k(\bx)]-\E_{\bx\sim \gaus_m}[\bH_k(\bx)] \|_2$. 
Suppose $\|\E_{\bx \sim A}[\bH_k(\bx)]-\E_{\bx\sim \gaus_m}[\bH_k(\bx)] \|_2>\nu$. 
Then it is easy to see that the polynomial 
\[
f=\left \langle \frac{\E_{\bx \sim A}[\bH_k(\bx)]-\E_{\bx\sim \gaus_m}[\bH_k(\bx)]}{\|\E_{\bx \sim A}[\bH_k(\bx)]-\E_{\bx\sim \gaus_m}[\bH_k(\bx)] \|_2},\bH_k(\bx)\right \rangle 
\]
satisfies the requirement, that $f$ is at most degree-$d$,
$\E_{\bx\sim\gaus_m}[f(\bx)^2]=1$ and $|\E_{\bx\sim A}[f(\bx)]-\E_{\bx\sim \gaus_m}[f(\bx)]|> \nu$ .
This leads to a contradiction.
Thus, $\|\E_{\bx \sim A}[\bH_k(\bx)]-\E_{\bx\sim \gaus_m}[\bH_k(\bx)] \|_2\leq \nu$.
Then to bound $\|\E_{\bx \sim A'}[\bH_k(\bx)]-\E_{\bx\sim A}[\bH_k(\bx)] \|_2$, 
let $\alpha=\pr_{\bx\sim A}[\bx\not \in \mathbb{B}^{m}(B)]$. We can write 
\begin{align*}
&\|\E_{\bx \sim A'}[\bH_k(\bx)]-\E_{\bx\sim A}[\bH_k(\bx)] \|_2\\
=&\left \|\frac{1}{1-\alpha}\E_{\bx \sim A}[\bH_k(\bx)\mathbf{1}(\bx\in \mathbb{B}^{m}(B))]
-\E_{\bx \sim A}[\bH_k(\bx)]\right \|_2\\
=&\left\|\frac{\alpha}{1-\alpha}\E_{\bx\sim A}[\bH_k(\bx)]-\frac{1}{1-\alpha}\E_{\bx\sim A}[\bH_k(\bx)\mathbf{1}(\bx\notin \mathbb{B}^m(B))]\right\|_2\\
\le&\frac{1}{1-\alpha}\left \|\E_{\bx \sim A}[\bH_k(\bx)\mathbf{1}(\bx\not\in \mathbb{B}^{m}(B))]\right \|_2
+\frac{\alpha}{1-\alpha}\left \|\E_{\bx \sim A}[\bH_k(\bx)]\right \|_2\\
\leq & \frac{1}{1-\alpha}\|\E_{\bx \sim A}[\bH_k(\bx)\mathbf{1}(\bx\not\in \mathbb{B}^{m}(B))]\|_2+
\frac{\alpha}{1-\alpha}\nu\; ,
\end{align*}
where the last inequality follows from 
$\|\E_{\bx \sim A}[\bH_k(\bx)]-\E_{\bx\sim \gaus_m}[\bH_k(\bx)] \|_2\leq \nu$
and $\E_{\bx\sim \gaus_m}[\bH_k(\bx)]=0$ for any $k>0$.
Since $B^d\geq c_1\left (2^{d/2}\sqrt{\frac{\Gamma(d+m/2)}{\Gamma(m/2)}}\right )$ ($c_1$ is at least a sufficiently large universal constant), we have $\alpha\leq c_2\left (2^{d/2}\sqrt{\frac{\Gamma(d+m/2)}{\Gamma(m/2)}}\right )B^{-d}\leq 1/2$. Also, we have $B^d\geq c_1\left (2^{d/2}\sqrt{\frac{\Gamma(d+m/2)}{\Gamma(m/2)}}\right )$ implies $B^2\geq m$. 
We will need the following fact:
\begin{fact} \label{fct:hermite-upper-bound-integration}
Let $\bH_k$ be the $k$-th Hermite tensor for $m$-dimension.
Suppose $\|\bx\|_2\geq m^{1/4}$,
then $\|\bH_k(\bx)\|_2=2^{O(k)}\|\bx\|_2^k$.
\end{fact}
\begin{proof}
For a degree-$k$ tensor $\bA$, we use $\bA^{\pi}$ to denote the matrix that
$\bA^{\pi}_{i_1,\cdots,i_k}=\bA_{\pi(i_1,\cdots,i_k)}$.
Notice $\|\bA\|_2=\|\bA^\pi\|_2$. Then from the definition of Hermite tensor, we have
\begin{align*}
    \bH(\bx)=\frac{1}{\sqrt{k!}}\sum_{t=0}^{\lfloor k/2\rfloor}\sum_{\text{Permutation $\pi$ of $[k]$}}
    \frac{1}{2^tt!(k-2t)!} \left ((-\bI)^{\otimes t}\bx^{\otimes (k-2t)}\right )^\pi\; .
\end{align*}
This implies that  
\begin{align*}
    &\|\bH_k(\bx)]\|_2\\
    =&\left \|\frac{1}{\sqrt{k!}}\sum_{t=0}^{\lfloor k/2\rfloor}\sum_{\text{Permutation $\pi$ of $[k]$}}
    \frac{1}{2^tt!(k-2t)!} \left ((-\bI)^{\otimes t}\bx^{\otimes (k-2t)}\right )^\pi\right \|_2\\
    \leq & \sum_{t=1}^{\lfloor k/2\rfloor}\frac{\sqrt{k!}}{2^t t!(k-2t)!} \max (\|\bI^{\otimes t}\|_2\|\bx\|_2^{k-2t},1 )\\
    =& \sum_{t=1}^{\lfloor k/2\rfloor}\frac{\sqrt{k!}}{2^t t!(k-2t)!} \max (m^{t/2}\|\bx\|_2^{k-2t},1 )\\
    =&  \|\bx\|_2^k\sum_{t=1}^{\lfloor k/2\rfloor}\frac{\sqrt{k!}}{2^t t!(k-2t)!}\; .
\end{align*}
One can see that the denominator is minimized when $t=k/2-O(\sqrt{k})$.
Then it follows that the sum is at most $2^{O(k)} \|\bx\|_2^k$.
\end{proof}
Since $B^2\ge m$, applying Jensen's inequality and Fact~\ref{fct:hermite-upper-bound-integration},
we have
\begin{align*}
    \|\E_{\bx \sim A}[\bH_k(\bx)\mathbf{1}(\bx\not\in \mathbb{B}^{m}(B))]\|_2
    &\leq \E_{\bx \sim A}[\|\bH_k(\bx)\|_2\mathbf{1}(\bx\not\in \mathbb{B}^{m}(B))]\\
    &\leq \E_{\bx \sim A}[2^{O(k)}\max(1,\|\bx\|_2^k)\mathbf{1}(\bx\not\in \mathbb{B}^{m}(B))]\\
    &\leq 2^{O(k)}\E_{\bx \sim A}[\|\bx\|_2^k\mathbf{1}(\bx\not\in \mathbb{B}^{m}(B))]\\
    &\leq 2^{O(k)}\int_{0}^\infty \pr[\|\bx\|_2\geq u\land \bx\not \in \mathbb{B}^{m}(B)] du^k\\
    &\leq 2^{O(k)}\int_{0}^\infty \left (2^{d/2}\sqrt{\frac{\Gamma(d+m/2)}{\Gamma(m/2)}}\right )\min(B^{-d},u^{-d}) du^k\\
    &\leq 2^{O(k)}\left (2^{d/2}\sqrt{\frac{\Gamma(d+m/2)}{\Gamma(m/2)}}\right )
    \left (\int_{0}^B B^{-d}du^k+\int_{B}^\infty u^{-d}du^{k}\right ) \\
    &\leq 2^{O(k)}\left (2^{d/2}\sqrt{\frac{\Gamma(d+m/2)}{\Gamma(m/2)}}\right )B^{-(d-k)}\; .
\end{align*}
Plugging it back gives
\begin{align*}
&\|\E_{\bx \sim A'}[\bH_k(\bx)]-\E_{\bx\sim A}[\bH_k(\bx)] \|_2\\
\leq &2^{O(k)}\left (2^{d/2}\sqrt{\frac{\Gamma(d+m/2)}{\Gamma(m/2)}}\right )B^{-(d-k)}+\left (1+O\left (2^{d/2}\sqrt{\frac{\Gamma(d+m/2)}{\Gamma(m/2)}}\right )B^{-(d-k)}\right )\nu\; ,
\end{align*}
for $k<d$.

Now we bound $\|\E_{\bx \sim A'}[\bH_k(\bx)]-\E_{\bx\sim \gaus_d}[\bH_k(\bx)]\|_2$
for $k\geq d$. 
For a degree-$k$ tensor $\bA$, we use $\bA^{\pi}$ to denote the matrix that
$\bA^{\pi}_{i_1,\cdots,i_k}=\bA_{\pi(i_1,\cdots,i_k)}$ and $\|\bA\|_2=\|\bA^\pi\|_2$.
For the case that $k\geq d$, 
from the definition of Hermite tensor, we have
\begin{align*}
    \bH(\bx)=\frac{1}{\sqrt{k!}}\sum_{t=0}^{\lfloor k/2\rfloor}\sum_{\text{Permutation $\pi$ of $[k]$}}
    \frac{1}{2^tt!(k-2t)!} \left ((-\bI)^{\otimes t}\bx^{\otimes (k-2t)}\right )^\pi\; .
\end{align*}
This implies that 
\begin{align*}
    &\|\E_{\bx \sim A'}[\bH_k(\bx)]-\E_{\bx\sim \gaus_m}[\bH_k(\bx)]\|_2\\
    =&\|\E_{\bx \sim A'}[\bH_k(\bx)]\|_2\\
    =&\left \|\frac{1}{\sqrt{k!}}\sum_{t=0}^{\lfloor k/2\rfloor}\sum_{\text{Permutation $\pi$ of $[k]$}}
    \frac{1}{2^tt!(k-2t)!} \left ((-\bI)^{\otimes t}\E_{\bx \sim A'}\left [\bx^{\otimes (k-2t)}\right ]\right )^\pi\right \|_2\\
    \leq & \sum_{t=1}^{\lfloor k/2\rfloor}\frac{\sqrt{k!}}{2^t t!(k-2t)!} \|\bI^{\otimes t}\|_2\left \|\E_{\bx \sim A'}\left [\bx^{\otimes (k-2t)}\right ]\right \|_2 \\
    \leq & \sum_{t=1}^{\lfloor k/2\rfloor}\frac{\sqrt{k!}}{2^t t!(k-2t)!} \|\bI^{\otimes t}\|_2\E_{\bx \sim A'}\left [\left \|\bx\right \|_2^{k-2t}\right ] \\
    = & \sum_{t=1}^{\lfloor k/2\rfloor}\frac{\sqrt{k!}}{2^t t!(k-2t)!} m^{t/2} B^{\max(k-2t-d,0)}\E_{\bx \sim A'}\left [\left \|\bx\right \|_2^{\min(k-2t,d)}\right ] \\   
    \leq  & B^{k-d}\sum_{t=1}^{\lfloor k/2\rfloor}\frac{\sqrt{k!}}{2^t t!(k-2t)!} m^{\max(0,2t-k+d)/4} \E_{\bx \sim A'}\left [\left \|\bx\right \|_2^{\min(k-2t,d)}\right ] \; ,    
\end{align*}
where the last inequality follows from $4(t/2)+\max(k-2t-d,0)=(k-d)+4(\max(0,2t-k+d)/4)$ and $B^2>m$ (which is implied by $B^d\geq c_1\left (2^{d/2}\sqrt{\frac{\Gamma(d+m/2)}{\Gamma(m/2)}}\right )$).
Since $A'$ is obtained by truncating any $\bx\not\in \mathbb{B}^m(B)$
and $\pr[\bx\not\in \mathbb{B}^m(B)]\leq 1/2$, 
we have $\E_{\bx \sim A'}\left [\left \|\bx\right \|_2^{\max(k-2t,d)}\right ]=O\left (\E_{\bx \sim A}\left [\left \|\bx\right \|_2^{\max(k-2t,d)}\right ]\right )$.
Thus, we have
\begin{align*}
    &\|\E_{\bx \sim A'}[\bH_k(\bx)]-\E_{\bx\sim \gaus_m}[\bH_k(\bx)]\|_2\\
    \leq & B^{k-d}\sum_{t=1}^{\lfloor k/2\rfloor}\frac{\sqrt{k!}}{2^t t!(k-2t)!} m^{\max(0,2t-k+d)/4} O\left (2^{\min(k-2t,d)/2}\sqrt{\frac{\Gamma(\min(k-2t,d)+m/2)}{\Gamma(m/2)}}\right )\\
    \leq & B^{k-d}O\left (2^{d/2}\sqrt{\frac{\Gamma(d+m/2)}{\Gamma(m/2)}}\right )\sum_{t=1}^{\lfloor k/2\rfloor}\frac{\sqrt{k!}}{2^t t!(k-2t)!}\; ,
\end{align*}
where the second inequality follows from $\max(0,2t-k+d)+\min(k-2t,d)=d$.
One can see the denominator is minimized when $t=k/2-O(\sqrt{k})$.
Then it follows that the sum is at most $2^{O(k)} B^{k-d} \left (2^{d/2}\sqrt{\frac{\Gamma(d+m/2)}{\Gamma(m/2)}}\right )$.

This completes the proof.

\end{proof}

\subsubsection{Proof of Fact~\ref{fct:hermite-upper-bound-medium-k}}\label{ssec:hermite-upper-bound-medium-k}
\begin{proof}
For a degree-$k$ tensor $\bA$, we use $\bA^{\pi}$ to denote the matrix that
$\bA^{\pi}_{i_1,\cdots,i_k}=\bA_{\pi(i_1,\cdots,i_k)}$.
Notice $\|\bA\|_2=\|\bA^\pi\|_2$. Then from the definition of Hermite tensor, we have
\begin{align*}
\|\bH(\bx)\|_2&\le\frac{1}{\sqrt{k!}}\sum_{t=0}^{\lfloor k/2\rfloor}\left\|\sum_{\text{Permutation $\pi$ of $[k]$}}
\frac{1}{2^tt!(k-2t)!} \left ((-\bI)^{\otimes t}\bx^{\otimes (k-2t)}\right )^\pi\right\|_2
\\&\le\frac{1}{\sqrt{k!}}\sum_{t=0}^{\lfloor k/2\rfloor}\frac{1}{2^tt!(k-2t)!}\sum_{\text{Permutation $\pi$ of $[k]$}}\|\bI\|_2^t\|\bx\|_2^{k-2t}\\&=\frac{1}{\sqrt{k!}}\sum_{t=0}^{\lfloor k/2\rfloor}\frac{k!}{2^tt!(k-2t)!}m^{t/2}\|\bx\|_2^{k-2t}\\
&\le\frac{m^{k/4}}{\sqrt{k!}}\sum_{t=0}^{\lfloor k/2\rfloor}\frac{\binom{k}{2}^{t}\|\bx\|_2^{k-2t}}{t!}
\\&\le 2^{k}m^{k/4}B^{k}k^{-k/2}\exp\left(\binom{k}{2}/B^2\right)\; .
\end{align*}
\end{proof}

\subsubsection{Proof of Fact~\ref{fct:hermite-upper-bound-large-k}}

\begin{proof}
We will need the following fact for the proof.
\begin{fact} [\cite{Krasikov2004}] \label{fct:hermite-upper-bound-large-k-helper}
Let $h_k$ be the $k$-th normalized Hermite polynomial. 
Then $\max_{t\in \R} h_k^2(t)e^{-t^2/2}=O(k^{-1/6} )$.
\end{fact}
First note that using Lemma \ref{clm:othor-tran}, we have for any orthogonal $\bB\in \R^{m\times m}$,
\[
\|\bH_k(\bx)\|_2=\|\bB^{\otimes k}\bH_k(\bx)\|_2=\|\bH_k(\bB\bx)\|_2\; .
\]
By taking the appropriate $\bB$, we can always have $\bB\bx=\|\bx\|_2\be_1$.
Therefore, without loss of generality, we can assume $\bx=t\be_1$ for some $t\in \R$.

Notice that for any entry $H_k(\bx)_{i_1,\cdots,i_k}$, let $j_\ell$ for $\ell\in {1,\cdots,m}$ be 
the number of times $\ell$ appears in $i_1,\cdots,i_k$.
Note that $\sum_\ell j_\ell=k$.
To bound the norm of $\bH_k(\bx)$ notice
\[
\bH_k(\bx)_{i_1,\cdots,i_k}=\binom{k}{j_1,\cdots,j_m}^{-1/2}\prod_{\ell=1}^m h_{j_\ell}(\bx_\ell)\; .
\]
Therefore,
\begin{align*}
\|\bH_k(\bx)\|_2^2=&\sum_{i_1,\cdots,i_k}\binom{k}{j_1,\cdots,j_m}^{-1}  \left(\prod_{\ell=1}^mh_{j_\ell}(\bx_\ell)\right)^2\\
=&\sum_{j_1,\cdots,j_m\text{ such that }\sum_\ell j_\ell=k} \left (\prod_{\ell=1}^m h_{j_\ell}(\bx_\ell)\right )^2\\
\le&\sum_{j_1,\cdots,j_m\text{ such that }\sum_\ell j_\ell=k} \prod_{\ell=1}^m O(\exp({\bx_\ell^2/2}))\\
\le & 2^{O(m)}\binom{k+m-1}{m-1} \exp({\|\bx\|_2^2/2})\; ,
\end{align*}
where the first inequality follows from Fact~\ref{fct:hermite-upper-bound-large-k-helper}.
Then we can take square root on both sides which gives what we want.
\end{proof}

\subsubsection{Proof of Lemma~\ref{lem:random-subspace-correlation-moment}}

\begin{proof}
To bound $\E_{\bV^\intercal\sim U(\orthor_{n,m})}[\|\bV\bu\|_2^k]$, by symmetry,
we can instead consider $\bV=[\be_1,\be_2,\cdots,\be_m]^\intercal$ and $\bu\sim U(\mathbb{S}^{n-1})$.
Then we have that
\begin{align*}
\E_{\bu\sim U(\mathbb{S}^{n-1})}\left [\|\bV\bu\|_2^k\right ]
=&\E_{\bu\sim U(\mathbb{S}^{n-1})}\left [\left (\|\bV \bu\|_2/\|\bu\|_2\right )^k\right ]
=\E_{\bu\sim U(\mathbb{S}^{n-1})}\left [\left (\|\bV \bu\|_2^2/\|\bu\|_2^2\right )^{k/2}\right ]\\
=&\E_{t\sim \mathrm{Beta}\left (\frac{m}{2},\frac{n-m}{2}\right )}\left [t^{k/2}\right ],
\end{align*}
where the last equation follows from the standard fact that 
if $X\sim\chi_{d_1}^2$ and $Y\sim\chi_{d_2}^2$ are independent 
then $\frac{X}{X+Y}\sim\mathrm{Beta}(d_1/2,d_2/2)$. 
By Stirling's formula, we can bound 
$\E_{t\sim \mathrm{Beta}\left (\frac{m}{2},\frac{n-m}{2}\right )}\left [t^{k/2}\right ]$ as follows:
\begin{align*}
\E_{t\sim \mathrm{Beta}\left (\frac{m}{2},\frac{n-m}{2}\right )}\left [t^{k/2}\right ]
&=\frac{1}{\mathrm{B}\left (\frac{m}{2},\frac{n-m}{2}\right )}\int_{t=0}^1 t^{(\frac{k+m}{2}-1)}(1-t)^{(\frac{n-m}{2}-1)}dt\\&=\frac{\mathrm{B}\left ( \frac{k+m}{2},\frac{n-m}{2}\right )}{\mathrm{B}\left (\frac{m}{2},\frac{n-m}{2}\right )}=\Theta\left (\frac{\Gamma\left (\frac{k+m}{2}\right )\Gamma\left (\frac{n}{2}\right )}
{\Gamma\left (\frac{k+n}{2}\right )\Gamma\left (\frac{m}{2}\right )}\right )\; .
\end{align*}
\end{proof}

\subsubsection{Proof of Corollary~\ref{col:random-subspace-correlation-moment}}\label{ssec:random-direction-correlation-moment-small-k}
\begin{proof}
 We will proceed by case analysis.
    We first consider the case where both $n$ and $m$ are even integers.
    By definition of the $\Gamma$ function, we have that
    \begin{align}\label{eq:even-small-k}
        \E_{\bV\sim U(\orthor_{n,m})}[\|\bV^{\intercal}\bu\|_2^k]
        =&\Theta\left (\frac{\Gamma\left (\frac{k+m}{2}\right )\Gamma\left (\frac{n}{2}\right )}
        {\Gamma\left (\frac{k+n}{2}\right )\Gamma\left (\frac{m}{2}\right )}\right )\notag
        =\Theta\left ( \frac{\left (\frac{n}{2}-1\right )\left(\frac{n}{2}-2\right)\cdots\left(\frac{m}{2}\right)}{\left (\frac{k+n}{2}-1\right )\left(\frac{k+n}{2}-2\right)\cdots\left (\frac{k+m}{2}\right )}\right )\notag\\
        =&\Theta\left ( \frac{\left (\frac{k+m}{2}-1\right )\left(\frac{k+m}{2}-2\right)\cdots\left (\frac{m}{2}\right )}{\left (\frac{k+n}{2}-1\right )\left(\frac{k+n}{2}-2\right)\cdots\left (\frac{n}{2}\right )}\right )
        \le O\left(\left(\frac{k+m}{k+n}\right)^{k/2}\right)\notag\\
        =&O\left(2^{k/2}\left(\frac{\max(k,m)}{n}\right)^{k/2}\right)
        \;.
    \end{align}
    For $m \le n^c < k$ for some constant $c\in(0,1)$, we have that
    \begin{align}\label{eq:large-even-k}
        &\E_{\bV\sim U(\orthor_{n,m})}[\|\bV^{\intercal}\bu\|_2^k]
        =\Theta\left (\frac{\Gamma\left (\frac{k+m}{2}\right )\Gamma\left (\frac{n}{2}\right )}
        {\Gamma\left (\frac{k+n}{2}\right )\Gamma\left (\frac{m}{2}\right )}\right )\notag\\
        =&\Theta\left ( 
        \frac{\left (\frac{n}{2}-1\right )\left(\frac{n}{2}-2\right)\cdots\left(\frac{m}{2}\right)}
            {\left (\frac{n^c+n}{2}-1\right )\left(\frac{n^c+n}{2}-2\right)\cdots\left (\frac{n^c+m}{2}\right )}
        \times
        \frac{\left (\frac{n^c+n}{2}-1\right )\left(\frac{n^c+n}{2}-2\right)\cdots\left (\frac{n^c+m}{2}\right )}
            {\left (\frac{k+n}{2}-1\right )\left(\frac{k+n}{2}-2\right)\cdots\left (\frac{k+m}{2}\right )}\right )\notag\\
        =& O(2^{n^c/2}n^{-(1-c)n^c/2})O\left (\left (\frac{n^c+n}{k+n}\right )^{(n-m)/2}\right )\notag\\
        =& \exp(-\Omega(n^c \log n))O\left (\left (\frac{n^c+n}{k+n}\right )^{(n-m)/2}\right ) \;,
    \end{align}
where the third equation follows from 
\eqref{eq:even-small-k} by taking $k=n^c$.

For the case where $n$ is even and $m$ is odd,
noting that $\Gamma(x+1/2)=\Theta(\sqrt{x}\Gamma(x))$, by Stirling's approximation we have that
\begin{align*}
\E_{\bV\sim U(\orthor_{n,m})}[\|\bV^{\intercal}\bu\|_2^k]
=\Theta\left (\frac{\Gamma\left (\frac{k+m}{2}\right )\Gamma\left (\frac{n}{2}\right )}
        {\Gamma\left (\frac{k+n}{2}\right )\Gamma\left (\frac{m}{2}\right )}\right )
        = \Theta\left(\frac{\sqrt{m}\Gamma\left(\frac{k+m-1}{2}\right)\Gamma\left(\frac{n}{2}\right)}{\sqrt{k+m}\Gamma\left(\frac{k+n}{2}\right)\Gamma\left(\frac{m-1}{2})\right)}\right)
        \;.
\end{align*}
By \eqref{eq:even-small-k}, we have that
\begin{align*}
\E_{\bV\sim U(\orthor_{n,m})}[\|\bV^{\intercal}\bu\|_2^k]
&=\Theta\left(\frac{\sqrt{m}\Gamma\left(\frac{k+m-1}{2}\right)\Gamma\left(\frac{n}{2}\right)}{\sqrt{k+m}\Gamma\left(\frac{k+n}{2}\right)\Gamma\left(\frac{m-1}{2})\right)}\right)=O\left(2^{k/2}\left(\frac{\max(k,m)}{n}\right)^{k/2}\right)\;.
\end{align*}
For $m\le n^c<k$ for some constant $c\in(0,1)$, by 
\eqref{eq:large-even-k} we have that
\begin{align*}
\E_{\bV\sim U(\orthor_{n,m})}[\|\bV^{\intercal}\bu\|_2^k]
&=\Theta\left(\frac{\sqrt{m}\Gamma\left(\frac{k+m-1}{2}\right)\Gamma\left(\frac{n}{2}\right)}{\sqrt{k+m}\Gamma\left(\frac{k+n}{2}\right)\Gamma\left(\frac{m-1}{2})\right)}\right)
\\&=\exp(-\Omega(n^c\log n))O\left (\left (\frac{n^c+n}{k+n}\right )^{(n-m+1)/2}\right )
\\&\le\exp(-\Omega(n^c\log n))O\left (\left (\frac{n^c+n}{k+n}\right )^{(n-m)/2}\right )\;.
\end{align*}
For the case where $n$ is odd and $m$ is even, by the Stirling approximation, we have that
\begin{align*}
\E_{\bV\sim U(\orthor_{n,m})}[\|\bV^{\intercal}\bu\|_2^k]
        &=\Theta\left (\frac{\Gamma\left (\frac{k+m}{2}\right )\Gamma\left (\frac{n}{2}\right )}
        {\Gamma\left (\frac{k+n}{2}\right )\Gamma\left (\frac{m}{2}\right )}\right )
        =\Theta\left(\frac{\sqrt{k+n}\Gamma\left(\frac{k+m}{2}\right)\Gamma\left(\frac{n+1}{2}\right)}{\sqrt{n}\Gamma\left(\frac{k+n+1}{2}\right)\Gamma\left(\frac{m}{2})\right)}\right) \;.
\end{align*}
By \eqref{eq:even-small-k}, we have that
\begin{align*}
\E_{\bV\sim U(\orthor_{n,m})}[\|\bV^{\intercal}\bu\|_2^k]
&=\Theta\left(\frac{\sqrt{k+n}\Gamma\left(\frac{k+m}{2}\right)\Gamma\left(\frac{n+1}{2}\right)}{\sqrt{n}\Gamma\left(\frac{k+n+1}{2}\right)\Gamma\left(\frac{m}{2})\right)}\right)
\\&= O\left(\sqrt{\frac{k+n}{n}}2^{k/2}\left(\frac{\max(k,m)}{n+1}\right)^{k/2}\right)
\\&=O\left(2^{k/2}\left(\frac{\max(k,m)}{n}\right)^{k/2}\right)\;.
\end{align*}
For $m\le n^c<k$ for some constant $c\in(0,1)$, by 
\eqref{eq:large-even-k}, we have that
\begin{align*}
\E_{\bV\sim U(\orthor_{n,m})}[\|\bV^{\intercal}\bu\|_2^k]
&=\Theta\left(\frac{\sqrt{k+n}\Gamma\left(\frac{k+m}{2}\right)\Gamma\left(\frac{n+1}{2}\right)}{\sqrt{n}\Gamma\left(\frac{k+n+1}{2}\right)\Gamma\left(\frac{m}{2})\right)}\right)
\\&=\exp(-\Omega(n^c\log n)O\left (\sqrt{\frac{k+n}{n}}\left (\frac{n^c+n+1}{k+n+1}\right )^{(n-m+1)/2}\right )
\\&\le\exp(-\Omega(n^c\log n))O\left (\left (\frac{n^c+n}{k+n}\right )^{(n-m)/2}\right )\;.
\end{align*}
For the case where
both $n$ and $m$ are odd integers, by the Stirling approximation, we have that
\begin{align*}
\E_{\bV\sim U(\orthor_{n,m})}[\|\bV^{\intercal}\bu\|_2^k]
        &=\Theta\left (\frac{\Gamma\left (\frac{k+m}{2}\right )\Gamma\left (\frac{n}{2}\right )}
        {\Gamma\left (\frac{k+n}{2}\right )\Gamma\left (\frac{m}{2}\right )}\right )
        =\Theta\left(\frac{\sqrt{m(k+n)}\Gamma\left(\frac{k+m-1}{2}\right)\Gamma\left(\frac{n+1}{2}\right)}{\sqrt{n(k+m)}\Gamma\left(\frac{k+n+1}{2}\right)\Gamma\left(\frac{m-1}{2})\right)}\right) \;.
\end{align*}
By \eqref{eq:even-small-k} we have that
\begin{align*}
\E_{\bV\sim U(\orthor_{n,m})}[\|\bV^{\intercal}\bu\|_2^k]
&=\Theta\left(\frac{\sqrt{m(k+n)}\Gamma\left(\frac{k+m-1}{2}\right)\Gamma\left(\frac{n+1}{2}\right)}{\sqrt{n(k+m)}\Gamma\left(\frac{k+n+1}{2}\right)\Gamma\left(\frac{m-1}{2})\right)}\right)
\\&= O\left(\sqrt{\frac{m(k+n)}{n(k+m)}}2^{k/2}\left(\frac{\max(k,m-1)}{n+1}\right)^{k/2}\right)
\\&=O\left(2^{k/2}\left(\frac{\max(k,m)}{n}\right)^{k/2}\right)\;.
\end{align*}
For $m\le n^c<k$ for some constant $c\in(0,1)$, by 
\eqref{eq:large-even-k} we have that
\begin{align*}
\E_{\bV\sim U(\orthor_{n,m})}[\|\bV^{\intercal}\bu\|_2^k]
&=\Theta\left(\frac{\sqrt{m(k+n)}\Gamma\left(\frac{k+m-1}{2}\right)\Gamma\left(\frac{n+1}{2}\right)}{\sqrt{n(k+m)}\Gamma\left(\frac{k+n+1}{2}\right)\Gamma\left(\frac{m-1}{2})\right)}\right)
\\&=\exp(-\Omega(n^c\log n)O\left (\sqrt{\frac{m(k+n)}{n(k+m)}}\left (\frac{n^c+n+1}{k+n+1}\right )^{(n-m+2)/2}\right )
\\&\le\exp(-\Omega(n^c\log n))O\left (\left (\frac{n^c+n}{k+n}\right )^{(n-m)/2}\right )\;.
\end{align*}
\end{proof}

\subsection{Omitted Proofs from Section~\ref{ssec:main-tail-bound}}\label{ssec:omit-main-tail-bound}

\subsubsection{Proof of Lemma~\ref{lem:high-degree-finite-chi-square}}\label{ssec:high-degree-finite-chi-square}
\begin{proof}
Notice that the distribution $D$ is a symmetric distribution.
Thus, if we can show that for $\bx\sim D$
the distribution $\|\bx\|_2^2$ is a continuous distribution, 
then $D$ is also a continuous distribution.
Note that the distribution $D$ can be thought of as generated by the following process.
To generate $\bx\sim D$, 
we first sample $\bt\sim A'$, $\bx'\sim \gaus_{n-m}$.
Let $\bu_1,\cdots,\bu_{n-m}$ be an orthonormal basis 
that spans the orthogonal complement of $\mathrm{span}(\bV)$.
We let $\bx=\sum_{i=1}^m t_i\bv_i+\sum_{i=1}^{d-1} \bx_i \bu_i$. 
Noting that $\|\bx\|_2^2=\|\bt\|_2^2+\|\bx'\|_2^2$, 
its distribution is the convolution sum of the distribution of $\|\bt\|_2^2$ 
and the $\chi^2_{n-m}$ distribution,
which is continuous.
Thus, as argued above, $D$ is a continuous distribution.

It remains to argue that $\chi^2(D,\mathcal{N}_n) = O_n(1)$.
We use $D_{r^2}$ to denote the distribution of $\|\bx\|_2^2$ above 
and $P_{r^2}$ to denote its pdf.
We use $S_{n-1}(r\mathbb{S}^{n-1})$ to denote the surface area of the $n$-dimensional sphere $r\mathbb{S}^{n-1}$ 
with radius $r$.
Then the pdf function of $D$ is 
\[D(\bx)=P_{r^2}(\|\bx\|_2^2)/S_{n-1}(\|\bx\|_2\mathbb{S}^{n-1})\; .\]
Similarly, the pdf function of $\gaus_n$ is 
\[\gaus_n(\bx)=\chi^2_n(\|\bx\|_2^2)/S_{n-1}(\|\bx\|_2\mathbb{S}^{n-1})\; .\]
Then we have
\begin{align*}
1+\chi^2(D,\gaus_n)&=\int_{\R^n}\frac{D(\bx)^2}{\gaus_n(\bx)} d\bx\\
&=\int_{\R^n}\frac{P_{r^2}(\|\bx\|_2^2)^2}{\chi^2_n(\|\bx\|_2^2)S_{n-1}(\|\bx\|_2\mathbb{S}^{n-1})} d\bx\\
&=\int_{0}^\infty\int_{r\mathbb{S}^{n-1}}
\frac{P_{r^2}(\|\bx\|_2^2)^2}{\chi^2_n(\|\bx\|_2^2)S_{n-1}(\|\bx\|_2\mathbb{S}^{n-1})} d\bx dr\\
&=\int_{0}^\infty\frac{P_{r^2}(r^2)^2}{\chi^2_n(r^2)} dr\\
&=\int_{0}^\infty \frac{P_{r^2}(r)^2}{2\sqrt{r}\chi^2_n(r)} dr\; .
\end{align*}
Thus, it remains to show that $\int_{0}^\infty \frac{P_{r^2}(r)^2}{2\sqrt{r}\chi^2_n(r)} dr=O_n(1)$.

We will first give a pointwise upper bound on $P_{r^2}$.
Notice that $D_{r^2}$ is the convolution sum of $D_{\|\bt\|_2^2}$ and the $\chi^2_{n-m}$ distribution,
where $\|\bt\|_2^2$ is inside $[0,n^2]$ since $A$ is supported on $\mathbb{B}^{m-1}(n)$. 
Thus, we can write 
\begin{align*}
P_{r^2}(r)= \int_{0}^{n^2} P_{\|\bt\|_2^2}(s) \chi^2_{n-m}(r-s) ds 
\leq  \max_{s\in [r-n^2,r]} \chi^2_{n-m}(s) \;.
\end{align*}
Plugging the pointwise upper bound back, we get
\begin{align*}
1+\chi^2(D,\gaus_n)
&\le\int_{0}^\infty\frac{\left (\max_{s\in [r-n^2,r]} \chi^2_{n-m}(s)\right )^2}{2\sqrt{r}\chi^2_n(r)} dr\\
&=O_n(1)\int_{0}^\infty\frac{\left (\max_{s\in [r-n^2,r]} s^{(n-m)/2-1}e^{-s/2}\right )^2}{r^{n/2-1/2}e^{-r/2}} dr\\
&=O_n(1)\int_{0}^\infty\frac{\max_{s\in [r-n^2,r]} s^{n-m-2}e^{-s}}{r^{n/2-1/2}e^{-r/2}} dr\\
&\leq O_n(1)\int_{0}^\infty r^{n/2-m-3/2}e^{-r/2}e^{n^2}dr\\
&= O_n(1)\int_{0}^\infty r^{n/2-m-3/2}e^{-r/2} dr\\
&=O_n(1)\Gamma(n/2-m-1/2)=O_n(1)\; .
\end{align*}
This completes the proof.
\end{proof}

\section{Omitted Details on Applications}\label{app:omit-application-proof}
In this section, we provide additional context on our applications
and provide the proofs of Theorems~\ref{thm:decodable-bound},~\ref{thm:anti-concentration-detection}, and~\ref{thm:SQ-periodic}.

\subsection{Proof of Theorem \ref{thm:decodable-bound}}
\begin{proof} 
This is a direct application of Theorem~\ref{thm:main-result}.
We will let the one-dimensional moment-matching distribution
be the distribution in Fact \ref{fct:decodable-moment-matching-distribution}.
Then any SQ algorithm distinguishing between
\begin{itemize}
    \item A standard Gaussian; and
    \item The distribution $\p_\bv^A$ for $\bv\sim U(\mathbb{S}^{n-1})$,
    where $A=\alpha\gaus(\mu,1)+(1-\alpha)E$ and $\mu=10c_d\alpha^{-1/d}$, 
\end{itemize}
with at least $2/3$ probability must require either 
a query of tolerance at most $O_d(n^{-((1-\lambda)/4-c) d})$
or $2^{n^{\Omega(1)}}$ many queries.
This proves the SQ lower bound for the NGCA testing problem.
However, it is not clear if there
is a simple and optimal reduction from list-decodable 
Gaussian estimation to the hypothesis testing problem.
Therefore, we will need to directly prove an SQ lower bound for the search problem.

Consider the following adversary for the search problem.
The adversary will let $X=\p_\bv^A$ for 
$\bv\sim U(\mathbb{S}^{n-1})$
be the input distribution,
and whenever possible,
the adversary will answer a query with $\E_{\bx\sim \gaus_n} [f(\bx)]$.
Given that the algorithm asks less than $2^{n^{\Omega(1)}}$ queries,
as we have shown in the proof of Theorem \ref{thm:main-result},
with $1-o(1)$ probability, 
the adversary can always answer $\E_{\bx\sim \gaus_n} [f(\bx)]$.
In such case, 
the algorithm will be left with $1-o(1)$ probability mess
over $\bv\sim U(\mathbb{S}^{n-1})$ that are equally likely.

Then we argue that no hypothesis can be close to more than $2^{-\Omega(n)}$
probability mass.
This can be done by upper bounding the surface area of a spherical cap
on a $n$-dimensional sphere, 
where the sphere is unit radius and the polar angle of the cap 
is a sufficiently small constant $\Phi$.
Note that the surface area of such a cap is 
$\Theta \left (I_{\sin^2 \Phi}\left (\frac{n-1}{2},\frac{1}{2}\right )\right )S_{n-1}(\mathbb{S}^{n-1})$,
where $I$ is the incomplete beta function and 
$S_{n-1}(\mathbb{S}^{n-1})$ is the surface area of the $n$ dimensional unit sphere.
Thus, it suffices to show that 
$I_{\sin^2 \Phi}\left (\frac{n-1}{2},\frac{1}{2}\right )=2^{-\Omega(n)}$.
Notice that, by its definition, we have
\begin{align*}
I_{\sin^2 \Phi}\left (\frac{n-1}{2},\frac{1}{2}\right )&=\frac{\int_0^{\sin^2 \Phi}t^{(n-3)/2}(1-t)^{-1/2}dt}
{B\left (\frac{n-1}{2},\frac{1}{2}\right )}\\
&=\frac{\int_0^{\sin^2 \Phi}t^{(n-3)/2}(1-t)^{-1/2}dt}
{B\left (\frac{n-1}{2},1\right )}\;
\frac{B\left (\frac{n-1}{2},1\right )}
{B\left (\frac{n-1}{2},\frac{1}{2}\right )}\\
&\leq O(1)\frac{\int_0^{\sin^2 \Phi}t^{(n-3)/2}(1-t)^{0}dt}
{B\left (\frac{n-1}{2},1\right )}\\
&=O(1)I_{\sin^2 \Phi}\left (\frac{n-1}{2},1\right )=O(1)(\sin^2 \Phi)^{(n-1)/2}=2^{-\Omega(n)}\; .
\end{align*}
Given that no hypothesis can be close to more than $2^{-\Omega(n)}$ probability mass, 
the only way to have any constant probability of success would be to return 
$2^{\Omega(n)}$ many hypotheses.
This completes the proof.
\end{proof}

\subsection{Proof of Lemma \ref{lem:anti-concentration-detection-moment-matching}}
 \begin{proof} 
We will take $E$ to be the distribution with density 
$E(t)=\gaus(t)+\mathds{1}(t\in [-1,1])p(t)$, 
where 
$p$ is a polynomial function that we truncate between $[-1,1]$.
In order to satisfy our requirements,
it will suffice to have $|p(t)|\leq 1/10$ for all $t\in [-1,1]$
and for each integer $0\leq i\leq d$,
\[\E_{t\sim A}[t^i]=(1-\alpha_d)\E_{t\sim E}[t^i]=(1-\alpha_d)\E_{t\sim \gaus}[t^i]+(1-\alpha_d)
\int_{-1}^{1}P(t)t^i dt=\E_{t\sim \gaus}[t^i]\; .\]
The second requirement is equivalently stated as follows: 
for each such $i$,
\[\int_{-1}^{1}P(t)t^i dt=\frac{\alpha_d}{1-\alpha_d}\E_{t\sim \gaus}[t^i]\;.\]

In order to satisfy these requirements,
we need the following fact from \cite{DK}.
\begin{fact} [Lemma 8.18 in \cite{DK}]
Let $C>0$ and $m\in \Z_+$. 
For any $a_0,a_1,\cdots,a_m\in \R$, there exists a unique degree
at most $m$ polynomial $p:\R\to \R$ such that for each integer 
$0\leq t\leq k$ we have that 
\[\int_{-C}^C p(x)x^tdx=a_t\; .\]
Furthermore, for each $x\in [-C,C]$ we have that $|p(x)|\leq O_m(\max_{o\leq t\leq m} a_t C^{-t-1})$.
\end{fact}
We apply the fact and take $C=1$.
This implies that there is such a polynomial with $|p(x)|\leq 1/10$ for sufficiently small $\alpha_d$ 
depending only on $d$.
\end{proof}

\subsection{Proof of Theorem \ref{thm:anti-concentration-detection}}
We provide a more detailed statement of Theorem \ref{thm:anti-concentration-detection} here.
\begin{theorem}[SQ Lower Bound for AC Detection] \label{thm:anti-concentration-detection-correct}
    There exists a function $f:(0,1/2)\to \N$ that $\lim_{\alpha\rightarrow 0} f(\alpha)=\infty$
    and satisfies the following.
    For any sufficiently small $\alpha\in (0,1/2)$, any SQ algorithm 
    that has access to a distribution that is either
    (i) a standard Gaussian; or
    (ii) a distribution that has at least $\alpha$ probability mass 
    in a $(n-1)$-dimensional subspace $V\subset \R^n$,
    and distinguishes the two cases with success probability at least $2/3$,
    either requires a query with error at most $O_\alpha(n^{-f(\alpha)})$,
    or uses at least $2^{n^{\Omega(1)}}$ many queries.
\end{theorem}
\begin{proof} [Proof of Theorem \ref{thm:anti-concentration-detection-correct}]
This is a direct application of Theorem~\ref{thm:main-result}.
We will let the one-dimensional moment-matching distribution
be the distribution in Lemma~\ref{lem:anti-concentration-detection-moment-matching},
where we will take $d$ only depends on $\alpha$ to be the largest integer such that 
$\alpha_d$ in Lemma~\ref{lem:anti-concentration-detection-moment-matching} satisfies $\alpha_d\geq\alpha$.
Notice that $d\rightarrow \infty$ as $\alpha\rightarrow 0$.
Taking $f(\alpha)=d/32$, it follows that
any algorithm distinguishing between
\begin{itemize}
    \item  A standard Gaussian; and
    \item  The distribution $\p_\bv^A$ for $\bv\sim U(\mathbb{S}^{n-1})$, 
    where 
    $A$ is the moment-matching distribution in Lemma~\ref{lem:anti-concentration-detection-moment-matching},
\end{itemize}
with at least $2/3$ probability must require either 
a query of tolerance at most $O_d(n^{-d/32})=O_\alpha(n^{-f(\alpha)})$
or $2^{n^{\Omega(1)}}$ many queries.
Notice that the distribution $\p_\bv^A$ has at least $\alpha$ probability mass 
resides inside the 
orthogonal complement of $\mathrm{span}(\bv)$, which is a $(n-1)$-dimensional subspace. 
Therefore, any SQ algorithm for solving the AC detection solves the hypothesis testing problem above. 
This completes the proof.
\end{proof}

\subsection{Proof of Theorem~\ref{thm:SQ-periodic}}
Let $G'_{s,\theta}$ be the probability measure obtained by rescaling $G'_{s,\theta}$ such that the total measure is one. 
We first show the following fact.
\begin{fact}\label{fct:discrete-gaussian-approx}
    For any polynomial $p$ of degree at most $k$ that $\E_{t\sim \gaus(0,1)}[p(t)^2]=1$, 
    $s$ that is at most a sufficiently small universal constant,
    $|\E_{t\sim G'_{s,\theta}}[p(t)]-\E_{t\sim \gaus(0,1)}[p(t)]|=k!2^{O(k)}\exp(-\Omega(1/s^2))$.
\end{fact}
\begin{proof}
Using Fact~\ref{fct:discrete-gaussian-moment}, 
we have that the total measure of $G_{s,\theta}$ is $1\pm \exp(-\Omega(1/s^2))$.
Therefore, for the rescaled $G'_{s,\theta}$, for any $k\in \N$, $s>0$ and all $\theta\in \R$,
\[
|\E_{t\sim \gaus(0,1)}[t^k]-\E_{t\sim G'_{s,\theta}}[t^k]|=k!\exp(-\Omega(1/s^2))\; .
\]
Then using the definition of Hermite polynomial, 
\begin{align*}
    \E_{t\sim \gaus(0,1)}[h_k(t)]-\E_{t\sim G'_{s,\theta}}[h_k(t)]
    =&\frac{1}{\sqrt{k!}}\sum_{t=0}^{\lfloor k/2\rfloor}
    \frac{k!}{2^tt!(k-2t)!}\left (\E_{t\sim \gaus(0,1)}[t^{k-2t}]-\E_{t\sim G'_{s,\theta}}[t^{k-2t}]\right )\\
    \leq &\left (k!\exp(-\Omega(1/s^2))\right )\sum_{t=0}^{\lfloor k/2\rfloor}
    \frac{\sqrt{k!}}{2^tt!(k-2t)!}t^{k-2t}\; .
\end{align*}
Notice that the denominator is minimized when $t=k/2-O(\sqrt{k})$.
Then it follows that the sum is at most $k!2^{O(k)}\exp(-\Omega(1/s^2))$.
Now let $p(t)=\sum_{i=0}^k \bw_i h_i(t)$. Since $\E_{t\sim \gaus(0,1)}[p(t)^2]=1$, it must be 
$\|\bw\|_2=1$ and it follows that $\|\bw\|_1\leq \sqrt{k}$.
Therefore, we have 
\[
    \E_{t\sim \gaus(0,1)}[p(t)]-\E_{t\sim G'_{s,\theta}}[p(t)]
    \leq \sum_{i=0}^k |\bw_i| |\E_{t\sim \gaus(0,1)}[h_k(t)]-\E_{t\sim G'_{s,\theta}}[h_k(t)]|
    =k!2^{O(k)}\exp(-\Omega(1/s^2))\; .
\]
\end{proof}
Now given Fact~\ref{fct:discrete-gaussian-approx}, we can apply our main result Theorem~\ref{thm:main-result}.
We will consider the following distribution distinguishing problem.
In both the null hypothesis case and the alternative hypothesis case, 
the algorithm is given a joint distribution $D$ of $(\bx,y)$ over $\R^n\times \R$.
In the null hypothesis case, we have $\bx\sim \gaus(0,\bI_n)$ and $y\sim U([-1,+1])$ independently.
While in the alternative hypothesis case, we have $\bx\sim \gaus(0,\bI_n)$ and $y=\cos(2\pi(\delta\langle \bw,\bx\rangle +\zeta))$ with noise $\zeta\sim \gaus(0,\sigma^2)$ as in the definition of learning periodic function.
Notice that any SQ algorithm that can always returns a hypothesis $h$ such that $\E_{(\bx,y)\sim D}[(h(\bx)-y)^2]=o(1)$ can also be easily used to distinguish the two cases.
Therefore to lower bound such SQ algorithms, it suffices for us to give an SQ lower bound for this distribution distinguishing problem.

Notice that in the alternative hypothesis, the distribution of $\bx$ conditioned on any value of $y$ is 
the hidden direction distribution $\p_\bw^A$ where $A$ is
$\int_{-\infty}^{\infty}f_y(\zeta) \frac{1}{2}\left (G'_{1/\delta,\frac{\arccos{(y)}}{2\pi}-\zeta}
+G'_{1/\delta,\frac{2\pi-\arccos{(y)}}{2\pi}-\zeta}\right )d \zeta$
and $f_y:\R\to\R$ is the PDF function of distribution of $\zeta$ conditioned on $y$.
Notice that this is a mixture
of $G_{s,\theta}$ with $s=1/\delta$ and different $\theta$.
Thus applying Theorem~\ref{thm:main-result} and Fact~\ref{fct:discrete-gaussian-approx} yields that any SQ algorithm for solving the distinguishing problem, either requires a query of error at most $O_{k}(n^{-((1-\lambda)/4-\beta) k})+k!2^{O(k)}\exp(-\Omega(1/s^2))$, or at least $2^{n^{\Omega(\beta)}}$ many queries
for $\beta>0$.
Since $k$ will have dependence on $n$, we will need to calculate the constant factor in $O_{k}(n^{-((1-\lambda)/4-\beta) k})$ which depends on $k$.
According to Proposition~\ref{prp:main-tail-bound}, plug in the factor gives
$\sqrt{k!}n^{-((1-\lambda)/4-\beta) k}+k!2^{O(k)}\exp(-\Omega(1/s^2))$.

For convenience, let $(1-\lambda)/4-\beta=\gamma$.
It only remains to choose the value of $k$ and $\gamma$ so that $k\leq n^{\lambda}$, $\gamma<(1-\lambda)/4$ and
$\sqrt{k!}n^{-\gamma k}+k!2^{O(k)}\exp(-\Omega(\delta^2))$ is minimized.
We will chose $k=n^{c''}$ and $\gamma=c''$ for $c'<c''<\min(2c,1/10)$.
Then the error tolerance here is 
$\sqrt{k!} n^{-\gamma k}+k!2^{O(k)}\exp(-\Omega(\delta^2))
\leq \exp(-k)+\exp({O(c''(\log n)n^{c''})+O(n^{c''})-\Omega(n^{2c})})
=\exp(-\Omega(n^{c''}))+\exp(-\Omega(n^{2c}))=\exp(-n^{c'})$.
The number of queries here is 
$2^{n^{\Omega(\beta)}}=2^{n^{\Omega((1-\lambda)/4-\gamma)}}=2^{n^{\Omega((1-c'')/4-c')}}=2^{n^{\Omega(1)}}$.
This completes the proof.

\subsection{SQ Lower Bounds as Information-Computation Tradeoffs} \label{app:ic}
We note that both aforementioned results 
can be viewed as evidence 
of information-computation tradeoffs for the application problems we discussed.

For the problem of list-decodable Gaussian mean estimation,
the information-theoretically optimal error is $\Theta(\log^{1/2} (1/\alpha))$
and is achievable with $\poly(n/\alpha)$ many samples 
(\cite{DKS18-list}; see, e.g., Corollary 5.9 and Proposition 5.11 of \cite{DK}).
The best known algorithm for this problem achieves $\ell_2$-error guarantee $O(\alpha^{-1/d})$
using sample complexity and run time $(n/\alpha)^{O(d)}$ 
(\cite{DKS18-list}, see Theorem 6.12 of \cite{DK}).
Notice that in order to achieve error guarantee even sub-polynomial in $1/\alpha$,
the above algorithm will need super-polynomial runtime and sample complexity.
Informally speaking, Theorem \ref{thm:decodable-bound}
shows that no SQ algorithm 
can perform list-decodable mean estimation with a sub-exponential in $n^{\Omega(1)}$ many queries, 
unless using queries of very small tolerance  --- that would require at least
super-polynomially many samples to simulate.
Therefore, it can be viewed as evidence supporting an inherent tradeoff between
robustness and time/sample complexity for this problem.

For the AC detection hypothesis testing problem,
the information-theoretically optimal sample complexity is $O(n/\alpha)$.
To see this, note that if the input distribution is a standard Gaussian,
then any $n$ samples will almost surely be linearly independent.
On the other hand, 
suppose that the input distribution has $\alpha$ probability mass in a subspace.
Then with $O(n/\alpha)$ many samples, with high probability, 
there will be a subset of $n$ samples all coming from that subspace,
which cannot be linearly independent.
However, our SQ lower bound suggests that no efficient algorithm 
can solve the problem with even $n^{\omega(1)}$ samples where the $\omega(1)$
is w.r.t. $\alpha\rightarrow 0$.
This suggests an inherent tradeoff between the sample complexity 
and time complexity of the problem.

For the problem of learning periodic functions,
the SQ lower bound given by our result will be larger than the algorithmic upper bound in~\cite{SZB21}(with sample complexity $O(n)$ and run-time $2^{O(n)}$).
However, the algorithms in~\cite{SZB21} are based on LLL lattice-basis-reduction which is not captured by the SQ framework; therefore, this does not contradict our SQ lower bound result.

\section{The Optimality of Parameters in Theorem \ref{thm:main-result}} \label{app:optimality-proof}
In this section, we show that the lower bound in Theorem \ref{thm:main-result}
is nearly optimal. Specifically, we construct an NGCA problem 
instance that can be solved by an SQ algorithm 
with a single query of tolerance $n^{-(d+2)/4}$ 
(as we explained, in most cases, we can take $\lambda, c$ 
arbitrarily close to $0$ and our lower bound 
on tolerance is $O_{m,d}(n^{-d/4})$).
We recall the definition of hypothesis testing version of NGCA:
For integers $n>m\ge1$ and a distribution $A$ supported on $\R^m$,
one is given access to a distribution $D$ such that either: (1) $H_0:D=\mathcal{N}_n$,
(2) $H_1: D$ is given by $\p_\bV^A$,
where $\bV\sim U(\mathbf{O}_{n,m})$.
The goal is to distinguish between these two cases $H_0$ and $H_1$.

To define our NGCA instance, we let $A=(1-2\epsilon)\mathcal{N}+\epsilon\delta_{B}+\epsilon\delta_{-B}+p\mathbf{1}[x\in[-1,1]]$ for some polynomial $p:\R\to\R$ of degree at most $d$,
where we assign probability mass $\epsilon$ on some points $B$ and $-B$ respectively.

In order to match the first $d$ moments of $A$ with the standard 
Gaussian, we require the following technical result.
\begin{lemma}[\cite{DK}]\label{lem:moment-matching}
Let $C>0$ and $d\in\mathbb{Z}_+$. For any $a_0,a_1,\ldots,a_d\in\R$, there exists a unique degree at most $d$ polynomial $p:\R\to\R$ such that for each integer $0\le t\le d$ we have that
\begin{align*}
\int_{-C}^Cp(x)x^tdx=a_t.
\end{align*}
Furthermore, for each $x\in[-C,C]$ we have that $|p(x)|\le O_d(\max_{0\le t\le d}|a_t|C^{-t-1})$.
\end{lemma}
In order to apply Lemma~\ref{lem:moment-matching}, we pick $C=1$ and $a_0=0$. For any $1\le t\le d$, we let $a_t=2\epsilon((t-1)!!-B^t)$ for even $t$ and $a_t=0$ for odd $t$.
By Lemma~\ref{lem:moment-matching}, there is a polynomial $p$ of degree at most $d$ such that $A$ exactly matches the first $d$ moments of the standard Gaussian.
Furthermore, for each $x\in[-1,1]$, we have that $|p(x)|\le O_d(\max_{0\le t\le d}|a_t|)$. We take $B=\sqrt{n}$ and $\epsilon=n^{-(d+2)/2}$. In this way, for any $x\in[-1,1]$, we have that $(1-2\epsilon)\mathcal{N}(x)\ge O_d(\max_{0\le t\le d}|a_t|)=O_d(\epsilon B^d)=O_d(1/n)$.
Therefore, $A$ is well-defined.

Let $F(\bx)=\mathrm{\textit{He}}_{(d+2)/2}\left(\frac{\|\bx\|_2^2-n}{\sqrt{n}}\right)$.
The SQ algorithm is simple.
\begin{itemize}
\item Ask the SQ oracle to obtain an estimate $v$ of $\E_{\bx\sim D}[F(\bx)]$ with tolerance $n^{-(d+2)/4}/4$.
\item If $v<n^{-(d+2)/4}/2$ then return $H_0$, otherwise return $H_1$.
\end{itemize}
We assume the dimension $n$ is sufficiently large.
To prove the correctness of our algorithm,
it suffices to show that $\E_{\bx\sim\p_\bv^A}[F(\bx)]\ge n^{-(d+2)/4}$.
Note that $\frac{\|\bx\|_2^2-n}{\sqrt{n}}$ is sufficiently close to the standard Gaussian, and we have that
\begin{align*}
\E_{\bx\sim\p_\bv^A}[F(\bx)]&=\E_{\bx\sim\p_\bv^A}\left[\mathrm{\textit{He}}_{(d+2)/2}\left(\frac{\|\bx\|_2^2-n}{\sqrt{n}}\right)\right]=\E_{y\sim\chi_{n-1},z\sim A}\left[\mathrm{\textit{He}}_{(d+2)/2}\left(\frac{y^2+z^2-n}{\sqrt{n}}\right)\right]\\&=\E_{z\sim A}\left[\left(\frac{z^2-1}{\sqrt{n}}\right)^{(d+2)/2}\right]=n^{-(d+2)/4}\E_{z\sim A}\left[(z^2-1)^{(d+2)/2}\right],
\end{align*}
where we apply the identity 
$\mathrm{\textit{He}}_k(x+y)=\sum_{\ell=0}^k\binom{k}{\ell}x^{k-\ell}\mathrm{\textit{He}}_\ell(y)$ in the third equation.
Since $A$ exactly matches the first $d$ moments with the standard Gaussian and $d$ is even, we have that
\begin{align*}
&\quad\E_{z\sim A}[(z^2-1)^{(d+2)/2}]-\E_{z\sim\mathcal{N}}[(z^2-1)^{(d+2)/2}]=\E_{z\sim A}[z^{d+2}]-\E_{z\sim\mathcal{N}}[z^{d+2}]\\&=2\epsilon B^{d+2}-2\epsilon\E_{z\sim\mathcal{N}}[z^{d+2}]+\int_{-1}^1p(x)x^{d+2}dx=2\epsilon B^{d+2}-2\epsilon(d+1)!!+\int_{-1}^1p(x)x^{d+2}dx\;.
\end{align*}
This gives that 
\begin{align*}
\E_{z\sim A}[(z^2-1)^{(d+2)/2}]&\ge2\epsilon(B^{d+2}-(d+1)!!)-\int_{-1}^1|p(x)|x^{d+2}dx\\&\ge2\epsilon(B^{d+2}-(d+1)!!)- O_d(\epsilon B^d)\int_{0}^C x^{d+2}dx\\&=2\epsilon(B^{d+2}-(d+1)!!)-O_d(1/n)\\&\ge1\;.
\end{align*}
Therefore, we conclude that $\E_{\bx\sim\p_\bv^A}[F(\bx)]=n^{-(d+2)/4}\E_{z\sim A}\left[(z^2-1)^{(d+2)/2}\right]\ge n^{-(d+2)/4}$.

\begin{remark}
{\em By the standard definition of an SQ algorithm, we need to pick a bounded function $F(\bx)$ for the SQ queries. To address this, we can truncate the value of $F(\bx)$ with $\frac{\|\bx\|_2^2-n}{\sqrt{n}}$ more than $M_d$, where $M_d$ is a parameter depending on $d$. Since $n$ is sufficiently large, it suffices to show that $\E_{x\sim\gaus}[\mathrm{\textit{He}}_k(x)\mathbf{1}[|x|\le M]]$ is sufficiently close to 0 for any $1\le k\le (d+2)/2$. This can be achieved since there is an $M = M_d$ such that $\left|\int_{-M}^M \gaus(x)\mathrm{\textit{He}}_k(x)dx\right|\le \frac{1}{O(d)^d},1\le k \le d$.}
\end{remark}

\end{document}